\documentclass[journal]{IEEEtran}
\ifCLASSOPTIONcompsoc
  \usepackage[nocompress]{cite}
\else
  \usepackage{cite}
\fi
\ifCLASSINFOpdf
\else
\fi
\usepackage{amsmath}
\usepackage{ragged2e}

\newcommand{\nodialogue}[1]{}

\usepackage{algorithm}

\usepackage{algorithmicx}

\usepackage[usenames,dvipsnames]{color}

\usepackage{enumitem}

\usepackage[noend]{algpseudocode}

\usepackage{subcaption}

\usepackage{mathtools,nccmath}
\usepackage{amsfonts}
\usepackage{amsthm}
\usepackage{amssymb}

\newtheorem{theorem}{Theorem}[section]

\newtheorem{lemma}[theorem]{Lemma}

\usepackage{booktabs}
\usepackage{makecell}

\usepackage{graphicx}
\usepackage{hyperref}

\usepackage{physics}

\newcommand{\ld}[1]{\textcolor{black}{#1}}
\newcommand{\ldr}[1]{\textcolor{black}{#1}}
\newcommand{\wb}[1]{\textcolor{black}{#1}}
\newcommand{\wbr}[1]{\textcolor{black}{#1}}
\newcommand{\wt}[1]{\textcolor{black}{#1}}
\newcommand{\wtr}[1]{\textcolor{black}{#1}}

\DeclareMathOperator*{\argmax}{arg\,max}
\DeclareMathOperator*{\argmin}{arg\,min}

\DeclarePairedDelimiterX{\klx}[2]{(}{)}{%
  #1\;\delimsize\|\;#2%
}
\newcommand{\kl}{KL\klx}

\newcommand{\e}[2]{{\mathbb E}_{#1}\left[ #2 \right]}

\begin{document}
%
\title{Bayesian Estimate of Mean Proper Scores for Diversity-Enhanced Active Learning}
%
%
%
%

\author{Wei Tan,
        Lan Du \IEEEmembership{Senior Member, IEEE},
        Wray Buntine
\IEEEcompsocitemizethanks{
\IEEEcompsocthanksitem Wei Tan and Lan Du are with Dept.\ of Data Science and AI,
Monash University, Clayton, Australia. 
E-mail: Lan.Du@monash.edu
\IEEEcompsocthanksitem Wray Buntine is now at VinUniversity, Hanoi, Vietnam.
\IEEEcompsocthanksitem Correspondence author: Lan Du.
}

\thanks{Manuscript received June 30th, 2022}}

%
%

\markboth{Journal of \LaTeX\ Class Files,~Vol.~XX, No.~X, August~20XX}%
{Shell \MakeLowercase{\textit{et al.}}: Bare Demo of IEEEtran.cls for Computer Society Journals}
%



\IEEEtitleabstractindextext{%
\begin{abstract}
\justifying
The effectiveness of active learning largely depends on the sampling efficiency of the acquisition function. Expected Loss Reduction (ELR) focuses on a Bayesian estimate of the reduction in classification error, and more general costs fit in the same framework. We propose Bayesian Estimate of Mean Proper Scores (BEMPS) to estimate the increase in strictly proper scores such as log probability or negative mean square error within this framework. We also prove convergence results for this general class of costs. 
To facilitate better experimentation with the new acquisition functions, we develop a complementary batch AL algorithm that encourages diversity in the vector of expected changes in scores for unlabeled data. 
To allow high-performance classifiers, we combine deep ensembles, and dynamic validation set construction on pretrained models, and further speed up the ensemble process with the idea of Monte Carlo Dropout. 
Extensive experiments on both texts and images show that the use of mean square error and log probability with BEMPS yields robust acquisition
functions and well-calibrated classifiers, and consistently outperforms the others tested. 
The advantages of BEMPS over the others are further supported by a set of qualitative analyses, where we visualise their sampling behaviour using data maps and t-SNE plots.

\end{abstract}

}

\maketitle

\IEEEdisplaynontitleabstractindextext

%
\IEEEpeerreviewmaketitle

\section{Introduction}\label{sec:introduction}

%
%
%
%

\IEEEPARstart{C}{lassification} is widely used and its performance has substantially improved due to the advent of deep learning.
However, a significant obstacle to its 
utilization is the paucity of labeled or annotated data.
\ldr{
The data annotation process performed by domain experts is often time-consuming, expensive 
and tedious, 
especially in the medical field,
due to the need for expertise combined with privacy concerns \cite{jimenez2018capsule}.}
Active Learning (AL) offers a solution by judiciously selecting the most informative data points for annotation, ensuring efficient use of a limited annotation budget \cite{settles2009active}.
The elegance of AL lies in that it primarily involves an expert assigning a class label. Yet, even for a simple task such as classification, a generally accepted theory for AL remains elusive.

\ldr{
Creating an appropriate acquisition function is the central aspect of AL.
In the AL literature,
acquisition functions are predominantly classified into two major categories:} uncertainty-based methods and diversity-based methods.
The former aims to select samples with the highest predictive uncertainty of a model. This uncertainty can be measured by metrics such as the least confident (LC) \cite{1642660, settles2009active}, entropy \cite{4563068, settles2009active}, mutual information \cite{KAJvAYG2019} and other methods. 
However, uncertainty-based methods are prone to adversarial examples \cite{pop2018deep} and noisy datasets, 
often resulting in 
\ldr{sampling redundant and even biased instances,}
thus degrading the training efficiency and the model performance.  
Conversely, diversity-based methods strive to acquire samples that best capture the data distribution of the entire unlabeled set \cite{yuan-etal-2020-cold,sener2017active}.
Their success largely depends on the representations of the unlabeled data used in k-Means clustering.
For example, 
ALPS and BERT-KM \cite{yuan-etal-2020-cold} compute surprisal embeddings from
the predictive word probabilities generated by a pretrained language model.
\ldr{Even though existing methods exhibit promising performance in different tasks,
we argue that diversity alone without considering uncertainty or vice versa is sub-optimal, particularly when the estimated uncertainty or the computed embeddings have nothing to do with the model performance \cite{schroder2020survey, ren2020survey}.}

Recently, some batch AL approaches focus on the uncertainty-diversity trade-off with hybrid methods \cite{ren2020survey}. 
One notable method is BADGE \cite{ash2019deep}, which uses the gradient representation generated by the last layer of the model to incorporate prediction uncertainty and sample diversity.
\ldr{While BADGE could achieve a balance between uncertainty and diversity without requiring adjustments to model hyperparameters, it faces the challenge of time-consuming distance calculations in the gradient space \cite{yuan-etal-2020-cold}.
Nevertheless, many hybrid approaches prioritize diversifying samples through the data representation of deep neural networks (DNNs), often neglecting prediction uncertainty.}

\ldr{In formulating AL systems using DNNs, especially in real-world scenarios, it is essential to consider 
a comprehensive list of factors in a holistic manner.}
For instance, one should account for the costs involved in both retraining a deep neural network (DNN) and using validation data for training \cite{siddhant-lipton-2018-deep}. The role played by transformer language models \cite{https://doi.org/10.48550/arxiv.2004.13138} is also essential, as is understanding the intricacies of batch mode AL, especially with an emphasis on diversity \cite{ren2020survey}. 
Moreover, one shouldn't overlook the necessary expertise and associated professional costs \cite{gao2020cost, zhdanov2019diverse}. 
When designing experiments, it's vital to consider the feasible size of the validation set.
\ldr{In many real-world scenarios, practitioners can be constrained to an expert annotation budget of, for instance, 1000 data samples, due to the inherent cost of experts.}
This limitation suggests that larger validation set sizes, as observed in \cite{ash2019deep,NEURIPS2019_84c2d486}, might not always be feasible.

\ld{With these desiderata, we propose a new set of acquisition functions based on the principle of strictly proper scoring rules and develop a hybrid approach that considers both uncertainty and diversity, for effective batch active learning using DNNs.}
\wb{As our first step, frameworks for cost reduction due to data acqusition are 
the expected value of sample information (EVSI) \cite{raiffa1961applied} and similarly
mean objective cost of uncertainty (MOCU) \cite{yoon2013quantifying}. We develop our AL model using strictly proper scoring
rules or Bregman divergences \cite{10.5555/1046920.1194902} within these frameworks. MOCU applied to errors required modifications
of the formula, resulting in WMOCU \cite{ZhaoICLR21} and SMOCU \cite{zhao2021bayesian}, in order to achieve convergence
and avoid getting stuck in error bands. }
\ld{
In contrast, strictly proper scoring rules avoid this problem by
utilizing expected scores generated by strictly convex functions.
The convexity of the scoring functions guarantees convergence of AL.}
The scoring rules can be adapted to different inference tasks (e.g., different utilities, precision-recall trade-offs, etc.). 
This property is preferable and beneficial for applications such as medical domains
where measures other than errors are relevant for an inference task.

This work represents a substantial extension of our previous NeurIPS 2021 conference paper BEMPS \cite{BEMPS_Wei_NEURIPS2011}, where we developed a family of novel acquisition functions based on strictly proper scoring rules for categorical variables, which generalize existing functions based on ELR and BALD. It was then 
instantiated with two scoring functions, namely CoreMSE and CoreLog.
The contributions of this extension are summarised as follows:
\begin{enumerate}[topsep=0pt, partopsep=0pt, noitemsep, leftmargin=*]
\item \textbf{Theoretical properties of BEMPS}: \wb{We provide a detailed proof demonstrating that active learning with the Bayesian estimate of mean strictly proper scores guarantees convergence.  
Moreover, we show that our acquisition function measures epistemic uncertainty, not aleatoric uncertainty.}
\item \textbf{Comprehensive quantitative experimental results}: 
\ld{Besides the results reported in \cite{BEMPS_Wei_NEURIPS2011},
we further validate BEMPS' performance on the image classification task, using two benchmark datasets, namely, MINST and CIFAR10.
Additionally, we perform an analysis of model calibration, measured by the expected calibration error.}
\wb{This also lets us explore the computational properties of our methods and two different ensembling techniques, deep ensembles and Monte-Carlo dropout.}
\item \textbf{Qualitative analysis of sampling behaviors}: 
 \ld{We further employ visualization techniques (i.e., Data Maps \cite{swayamdipta2020dataset} and t-SNE \cite{JMLR:v9:vandermaaten08a} ) to delve into the sampling behavior of different AL methods, in order to gain insight into the underlying mechanism that drives the acquisition of samples from different regions (i.e., the easy-to-learn, the ambiguous and the hard-to-learn regions) and along the decision boundaries within the sample space.}

\end{enumerate}

\section{Related Work}

Active learning has been researched extensively since the emergence of deep learning,
generating significant interest. 
As elucidated by previous studies\cite{settles2009active, ren2020survey, schroder2020survey, ZhaoICLR21, zhao2021bayesian, NEURIPS2021_50d2e70c, NEURIPS2021_4afe0449}, 
there is no effective universal AL method. 
Instead, researchers often rely on heuristics tailored to their specific tasks.
\ldr{In this context, we review existing acquisition functions (aka query strategies) 
proposed in some recent AL models that are most related to ours.
These models cover a spectrum of perspectives, including
uncertainty-based, diversity-based and hybrid methods.
Readers interested in a thorough discourse on contemporary strategies 
are directed to \cite{ren2020survey} for a comprehensive discussion.}

\textbf{Uncertainty-based methods:} 
The proposed BEMPS \cite{BEMPS_Wei_NEURIPS2011}, as a general Bayesian model for acquisition functions, 
quantifies the model uncertainty using the theory of (strictly) proper scoring rules for categorical
variables \cite{doi:10.1198/016214506000001437}.
Thus, we first review existing acquisition functions 
proposed in some recent uncertainty-based AL models that are most related to ours.
One common and straightforward heuristic often used in AL is maximum entropy \cite{zhu2009active,zhu2012uncertainty,wang2014new,gal2017deep},
where one chooses samples that maximize the predictive entropy.
Similarly, Bayesian AL by disagreement (BALD) \cite{Houlsby2011} and its batch version (i.e.,
BatchBALD) \cite{KAJvAYG2019} instead compute the mutual information between the model predictions
and the model parameters, which indeed chooses samples that maximize the decrease in
expected entropy \cite{NEURIPS2019_84c2d486}.
\wb{Improving on ELR \cite{RoyMcC2001}, 
MOCU with minimum errors was extended by WMOCU using weights \cite{ZhaoICLR21}, also obtaining a theoretical guarantee of convergence.}
WMOCU directly minimizes the objective uncertainty impacting classification by tuning a hyperparameter of the weighting function.
\ld{
Published alongside WMOCU, SMOCU \cite{zhao2021bayesian} achieves convergence through the utilization of a log-sum-exp function to approximate the max function in computing the OBC error. This adaptation renders the MOCU function strictly concave.}
Without modifying the EVSI formula, 
BEMPS uses proper scores directly in the expected value of evidence.
\ld{
This requires no additional manipulation of the formula needed by the MOCU-based methods to achieve convergence.}

\textbf{Diversity-based and hybrid methods:} 
Query strategies considering just uncertainty do not work well in a batch setting because similar samples can be acquired in one batch \cite{WangZengmao2016Abal, ren2020survey,ma2021active}.
To overcome this problem, there have been many AL methods that
achieve batch diversity by acquiring samples that are both informative and diverse, such as \cite{Nguyen2004,yin2017deep, PengLiu2017ADLf,ash2019deep, yuan-etal-2020-cold,zhdanov2019diverse, shi-etal-2021-diversity}.
For instance,
BADGE \cite{ash2019deep} and ALPS \cite{yuan-etal-2020-cold} 
are the two recent AL methods focusing on batch diversity. 
BADGE uses gradient embeddings of unlabeled samples as inputs of $k$-MEANS++ to select a set of diverse samples, which relies on fine-tuning pretrained models.
Inspired by BADGE,
ALPS uses surprisal embeddings computed from
the predictive word probabilities generated by a masked language model for cold-start AL.
Whereas our BEMPS computes an embedding vector for a candidate sample data by the expected change in the proper scores induced by acquiring a label for the data.  This is a vector because it is computed for each unlabeled sample in a pool, answering the question ``how are proper scores expected to change across the pool if we acquire this one label?"


\textbf{Ensemble methods:} Ensemble methods combine multiple distinct models to improve generalization performance. 
They are now widely utilized in machine learning \cite{he2016deep, krizhevsky2012imagenet}. 
For example,
\cite{Lakshminarayanan2017, pawlowski2017efficient} uses ensembles to estimate the uncertainty of DNN predictions in the context of outlier detection and reinforcement learning.
In recent years, ensembles are also used to obtain better uncertainty estimates with deep learning, including deep ensembles \cite{Lakshminarayanan2017} and Monte-Carlo dropout (MC-Dropout) \cite{gal2016dropout,KAJvAYG2019,pop2018deep}.
Deep ensembles combine the benefits of both deep learning and ensemble learning to produce a model with superior generalization performance and uncertainty estimation, but they are very expensive to train and evaluate.
MC-Dropout is a well-known ensemble method that is less costly but less reliable \cite{durasov2021masksembles}.
BatchBALD utilizes the MC-dropout technique to approximate inference within Bayesian Neural Networks (BNNs). This approach aids in estimating uncertainty, enabling effective scalability to handle high-dimensional image inputs.

\textbf{Validation setup:}
In experiments, 
using a validation set to train deep learning models in active learning is not uncommon, for instance, for early-stopping.
Some existing methods assume that there
is a large validation set available a priori
\cite{ash2019deep, KAJvAYG2019, gal2017deep},
which means the cost of labeling the validation set is not factored into the labeling budget.
We argue that the availability of a separate validation set is impractical
in real-world AL scenarios.
Although
\cite{yuan-etal-2020-cold} uses fixed epochs to train the classifier without a validation set to save the cost, 
 the classifier could either be under-fit or over-fit.
 We instead use a dynamic approach to generate alternative validation sets
 from the ever-increasing labeled pool after each iteration.

\section{Bayesian Estimate of Mean Proper Scores}
\label{bemps}

We first review the general Bayesian model for acquisition functions, 
including ELR, MOCU and BALD;
\ld{and provide the preliminaries of proper scoring rules.}
We then develop BEMPS 
a new uncertainty quantification framework with a theoretical foundation  based on 
strictly proper scoring rules \cite{doi:10.1198/016214506000001437}.

\subsection{Expected Loss Reduction}
The ELR method attempts to quantify the generalization error that is expected to decrease when one new sample is added to the labeled dataset. \wb{
MOCU \cite{yoon2013quantifying} is similar to a broader category of scoring systems known as the value of information \cite{raiffa1961applied}.  Most relevant for us is the concept of the expected value of sample information (EVSI) \cite{raiffa1961applied}. This concept explores how the cost measure can be anticipated to improve upon acquiring a single sample point. Our distinct contribution lies in adopting a proper scoring rule as the designated cost metric within the EVSI framework.
}

\begin{table}[t]
  \caption{Notation}
  \label{tab:EER-notions}
  \centering
  \small
  {\color{black}
  \begin{tabular}{lp{0.6\linewidth}}
    \toprule
     Notation & Description  \\
    \cmidrule(r){1-2}
      $\vb{L}$ & a set of labeled data \\
    $\vb{U}$ & a set of unlabeled data \\
    $\vb{X}$ & a set of unlabeled data used to estimate the expected score change  \\
     $\vb{x}$ & a unlabeled sample from $U$  \\
     $y$ & the class label \\
     $\vb{x}'$ & a unlabeled sample sampled from $X$ \\
     $\vb*{\theta}$ & the model parameters \\
     $\vb{\Theta}$ & a set of model parameters \\ 
        $\mathbb{E}_{\mu}[f(\mu,\nu)]$ & the expectation of a function $f(\mu,\nu)$ 
        with respect to a random variable $\mu$ \\
    $p(y|\vb*{\theta}, \vb{x})$ & the probability  of label $y$ for an instance $x$ on running model $\vb*\theta$ \\
    $p(\vb*\theta | \vb{L})$ & the probability of $\vb*{\theta}$ based on learning from  $\vb{L}$  \\
   $p(\vb*\theta | \vb{L}, (\vb{x}, y))$ & the probability of $\vb*{\theta}$ based on the learning from $\vb{L}$ and an additional labeled instance $(\vb{x}, y)$\\
    $p(y | \vb{L}, \vb{x})$ & the Bayesian optimal estimate of label $y$ based on $\vb*\theta$ after learning $p(\vb*\theta| \vb{L})$, i.e., $\e{p(\vb*\theta | \vb{L})}{p(y|\vb*{\theta}, \vb{x})}$\\
    $Q(\cdot)$ & the score that measures the quality of the posterior of a model \\
    $\Delta Q(\vb{x} | \vb{L})$ &  the expected score change by acquiring the label of $\vb{x}$, 
   i.e., the acquisition function in AL\\
      $S(\cdot)$ & a strictly proper scoring function \\
     $G(\cdot)$ & a strictly convex function computed as the expectation of $S(\cdot)$ \\
    \bottomrule
  \end{tabular}%
  }
\end{table}

Suppose models of our interest are parameterised by parameters $\vb*{\theta} \in \vb{\Theta}$,
$\vb{L}$ indicates labeled data,
probability of label $y$ for data $\vb{x}$ is given by $p(y|\vb*{\theta},\vb{x})$,
and $p(\cdot|\vb*{\theta},\vb{x})$ presents a vector of label probabilities.
With a fully conditional model, 
the posterior of $\vb*{\theta}$
is unaffected by unlabeled data, which means 
\begin{math}
  p(\vb*{\theta}|\vb{L},\vb{U})=p(\vb*{\theta}|\vb{L})
\end{math}
for any unlabeled data $\vb{U}$.
Moreover, 
we assume without loss of generality that this model family is well-behaved in a statistical sense, 
so the model is identifiable.
The ``cost" of the posterior $p(\vb*{\theta}|\vb{L})$ can be measured
by some functional $Q(p(\vb*{\theta}|\vb{L}))$, denoted $Q(\vb{L})$ for short, where $Q(\vb{L})\ge 0$ and $Q(\vb{L})=0$
when some convergence objective has been achieved.  For our model, this is
when $p(\vb*{\theta}|\vb{L})$ has converged to a point mass at a single model.  
A suitable objective function is to measure the expected decrease in $Q(\cdot)$ 
due to acquiring the label for a data point $\vb{x}$.
The corresponding acquisition function for AL is formulated 
\cite[Eq~(1)]{Houlsby2011}, \cite[Eq~(3)]{zhao2021bayesian} as
\begin{equation}
    \label{eq-dQ}
\Delta Q(\vb{x}|\vb{L}) = Q(\vb{L}) - \mathbb{E}_{p(y|\vb{L},\vb{x})} \big[Q(\vb{L}\cup \{(\vb{x},y)\}) \big],
\end{equation}
whereas for ELR the expression is split over an inequality sign \cite[Eq~(2)]{RoyMcC2001}.
It estimates how much the cost is expected to reduce when a new data point $x$ is acquired.  
Since the true label for the new data $\vb{x}$ is unknown a prior, 
we have to use expected posterior proportions
from our model, $p(y|\vb{L},\vb{x})$ to estimate the likely label.
For BALD \cite{Houlsby2011} using Shannon's entropy, $Q_I(\vb{L}) = I(p(\vb*{\theta}|\vb{L}))$, which measures uncertainty in the parameter space and thus has no strong relationship to actual errors \cite{ZhaoICLR21}.
MOCU\cite{ZhaoICLR21} is equivalent to ELR that uses a Bayesian regret given by the expected loss difference between the
optimal Bayesian classifier and the optimal classifier:
\begin{flalign}
&Q_{MOCU}(\vb{L}) = \\ \nonumber
&\mathbb{E}_{p(\vb{x})} \Big[ 
          \min_{y} (1-p(y|\vb{L},\vb{x})) 
          -  \mathbb{E}_{p(\theta|\vb{L})}  \big[\min_{y} (1-p(y|\vb*{\theta},\vb{x})) \big] \Big]. \label{eq-q-mocu}
\end{flalign} 

WMOCU uses a weighting function defined by Eq (11) in \cite{ZhaoICLR21} to have a more amenable definition of $\Delta Q(\vb{x|L})$ than the MOCU method. 
Although WMOCU guarantees $\Delta Q(\vb{x|L})$ converging to the optimal classifier (under minimum errors) according to the $Q(\vb{L})$ with the strictly concave function by Eq (15) in \cite{ZhaoICLR21}, 
the optimal approximation of the convergences can only be solved by controlling a hyperparameter of the weighting function manually. To allow theoretical guarantees of convergence under more general loss functions, 
we propose a different definition for $Q(\vb{L})$ based on strictly proper scoring rules.

\subsection{Proper Scoring Rules}
\label{ssct-PSR}
\ld{Proper scoring rules assess the quality of probabilistic forecasts and offer a measure of predictive uncertainty. 
They assign a numerical
score based on the predictive distribution and the predictive outcome,
favouring more accurate and calibrated forecasts over less accurate ones. We focus on scoring functions in which a higher numeric score indicates a better performance.
In this paper, we are interested in the quality of the probabilistic prediction of categorical variables, where proper scoring rules are often used in training a classification algorithm.
Readers interested in a more detailed discussion of proper scoring cores can refer to \cite{savage1971elicitation,eorms0749,merkle2013choosing,dawid2014theory, doi:10.1198/016214506000001437}.
}

\ld{
Let us consider the prediction of a categorical variable $y$ with the sample space $\Omega = \{0, \dots, K-1\}$.
In this context, a scoring function $S(p(\cdot | \vb*{\theta}, \vb{x}), y)$ 
evaluates the quality of the predictive distribution $p(\cdot | \vb*{\theta}, \vb{x})$ 
over $\Omega$
relative to 
an observed event $y | \vb{x} \sim q(y | \vb{x})$ where $q(\cdot)$ 
denotes the true distribution on $(y, \vb{x})$, 
which $p(\cdot)$ tries to estimate.
Drawing upon the insights of  \cite{savage1971elicitation} and \cite{doi:10.1198/016214506000001437},
it becomes apparent that a scoring rule is intricately linked to
a tangent line of a convex, real-valued function $G(p(\cdot))$  
at $p(\cdot)$ 
evaluated at $q(\cdot)$.
The derivation of the scoring rule pertaining to an observed event and
the expected scoring rule  follows in the subsequent manner:
\begin{eqnarray}
    S(p(\cdot), y)
     &=& G(p(\cdot)) + \langle\grad{G(p(\cdot))}, (\delta_y - p(\cdot)) \rangle \,,
     \\
    S(p(\cdot), q(\cdot))    &=& \e{y\sim q(\cdot)}{S(p(\cdot), y)}
    \label{eq:prop:exp1}\\
    &=& G( p(\cdot)) +
    \langle\grad{G( p(\cdot))},\, (q(\cdot) - p(\cdot)) \rangle\,, 
    \label{eq:prop:exp2}
\end{eqnarray} 
where $\grad{G(p(\cdot))}$ is a subgradient of $G(\cdot)$ at the point 
$p(\cdot)$.
Note that the expectation of a scoring rule according to the
supplied probability
$p(\cdot)$ takes a simple form  $G(p(\cdot))$, 
the proof of which is straightforward by 
replacing $y \sim q(\cdot)$ with  $y \sim p(\cdot)$ in Eq~\eqref{eq:prop:exp1};
and the second term in Eq~\eqref{eq:prop:exp2} will be zero.}

\ld{
The scoring rule is proper if $S(q(\cdot), q(\cdot)) \geq S(p(\cdot), q(\cdot))$ 
for all $p(\cdot), q(\cdot) \in \mathcal{P}$\footnote{$\mathcal{P}$ is a convex set of probability measures on $(\mathcal{A}, \Omega)$, where $\mathcal{A}$ is a $\sigma$-algebra of subsets of the sample space $\Omega$ (i.e., classes in classification).}, and  
it is strictly proper if the expected score is minimized if and only if $p(\cdot) = q(\cdot)$ \cite{doi:10.1198/016214506000001437}. 
With $G(\cdot)$ being a convex function, the proof of the inequality is derived as follows:
    \begin{eqnarray}
        S(q(\cdot), q(\cdot)) &=& G(q(\cdot)) \nonumber\\
            &\geq& G( p(\cdot)) + \langle\grad{G( p(\cdot))},\, (q(\cdot) - p(\cdot)) \rangle\nonumber\\
            &=& S(p(\cdot), q(\cdot)) \,,
    \end{eqnarray}
where the inequality holds because $G(\cdot)$ is a convex function that is always above its tangent line computed at $p(\cdot)$.
In the case where $G(\cdot)$ exhibits strict convexity, the previously outlined proper scoring rule transitions into a strictly proper scoring rule, as indicated by  \cite{doi:10.1198/016214506000001437}. }

A strictly proper scoring rule has the behaviour that in the limit of
infinite labeled data $\vb{L}_{n}$, as $n\rightarrow \infty$, the average score
$\frac{1}{n} \sum_{(\vb{x},y)\in \vb{L}_n} S(p(\cdot|\vb*{\theta},\vb{x}),y)$ has a unique maximum for $\vb*{\theta}$ at the
``true" model (for our identifiable model family).
\ld{
With these characteristics, one can tune $G(\cdot)$ for different tasks \cite{merkle2013choosing,dawid2014theory}.
For instance, \cite{Lakshminarayanan2017} used Gibbs inequality to prove that the scoring function corresponding to maximizing likelihood is indeed a proper scoring rule.
In the following sections, we will develop two acquisition functions using the Brier score and the logarithmic score that correspond to the mean square error and the Kullback–Leibler divergence.}

\subsection{Strictly Proper Scores for Active Learning}

With strictly proper scoring rules, 
we develop a generalized class of acquisition functions built using the posterior (i.e., w.r.t. $p(\vb*{\theta}|\vb{L})$) expected difference between the score for the Bayes optimal classifier and the score for the ``true" model.
This is inherently Bayesian due to the use of $p(\vb*{\theta}|\vb{L})$.
\begin{flalign} 
\begin{split}
&Q_S(\vb{L}) = \\ &\,\,\,\mathbb{E}_{ p(\vb{x}) p(\vb*{\theta}|\vb{L}) } \Big[
\mathbb{E}_{ p(y|\vb*{\theta},\vb{x}) } \big[
  S(p(\cdot|\vb*{\theta},\vb{x}),y)  - S(p(\cdot|\vb{L,x}),y)  \big] \Big] 
\end{split}\label{eq-QSS}, \\
&Q_S(\vb{L})  = \mathbb{E}_{ p(\vb{x}) p(\vb*{\theta}|\vb{L}) } \big[B(p(\cdot|\vb{L,x}),p(\cdot|\vb*{\theta},\vb{x})) \big] \label{eq-QB} , \\
&Q_S(\vb{L})  =\mathbb{E}_{ p(\vb{x})}\Big[ \mathbb{E}_{ p(\vb*{\theta}|\vb{L}) }\big[G(p(\cdot|\vb*{\theta},\vb{x}))\big] 
 - G( p(\cdot|\vb{L,x}))\Big] \label{eq-QS} ,\\
\begin{split}
 &\Delta Q_S(\vb{x|L}) = \\ 
& \,\,\, \mathbb{E}_{ p(\vb{x'})} \Big[
  \mathbb{E}_{ p(y|\vb{L,x}) } \big[G( p(\cdot|\vb{L},(\vb{x},y),\vb{x'})) \big]
  - G( p(\cdot|\vb{L,x'})) \Big] 
\end{split} \label{eq-DQS} ,
\end{flalign}
where $\vb{x}'$ is a unlabeled sample, different from $\vb{x}$, which 
is used to estimate the expected score change $\Delta Q_x(\vb{x}|\vb{L})$.






The $Q_S(\vb{L})$ has three equivalent variations, 
one for an arbitrary strictly proper scoring rule $ S(q(\cdot),y)$ (Eq~\eqref{eq-QSS}),
one for a corresponding Bregman divergence $B(\cdot,\cdot)$ (Eq~\eqref{eq-QB})
and the
third for an arbitrary strictly convex function $G(\cdot)$
(Eq~\eqref{eq-QS}).
Their connections are given in \cite{doi:10.1198/016214506000001437}.
The form for $\Delta Q_S(\vb{x|L})$ corresponds to \cite[Equation~(6)]{zhao2021bayesian}. using the same simplification.

\begin{lemma}[Properties of scoring]\label{thm-nn}
In the context of a fully conditional classification model $p(y|\vb*\theta,\vb{x})$,
the $Q_I(\vb{L})$, $Q_{S}(\vb{L})$, $\Delta Q_I(\vb{x|L})$, $\Delta Q_{S}(\vb{x|L})$ as defined above are all non-negative.
\end{lemma}

\begin{proof}
$Q_I(\vb{L})\ge 0$ by definition of entropy.
Now an identity for entropy is that
$\mbox{I}\left( p( A|B,C)\right) \leq \mbox{I}\left(p(A|B)\right)$,
which means given more evidence $C$, the conditional entropy of
$A$ cannot increase.
So the decrease in log volume is never negative.
This means $\Delta Q_I(\vb{x|L})\ge 0$.
\ld{
The result for $Q_{S}(\vb{L})$ follows directly
by the definition of a scoring rule.}
\wb{
The expression $S(q(\cdot), q(\cdot)) - S(p(\cdot), q(\cdot))$ can be matched inside the outer square brackets of 
Eq~\eqref{eq-QSS}, and it is non-negative for a proper scoring rule.
}

For  $\Delta Q_{S}(\vb{x|L})$ we work as follows.
Start with Equation~\ref{eq-DQS}
and reverse back in the scores:
\begin{flalign} \label{eq:proof8}
&\mathbb{E}_{p(y | \vb{L,x} )} \Big[ \mathbb{E}_{ p(\vb{x'}) p(y'|\vb{L},(\vb{x},y),\vb{x'}) } \big[
      S(p(\cdot|\vb{L},(\vb{x},y),\vb{x'}),y') \big] \Big] 
      \nonumber \\ 
      &\quad- \mathbb{E}_{ p(\vb{x'}) p(y'|\vb{L,x'}) } \big[ S(p(\cdot|\vb{L,x'}),y') \big]      \nonumber  \\
&=\mathbb{E}_{p(y | \vb{L,x} )} \Big[ \mathbb{E}_{ p(\vb{x'})  p(y'|\vb{L},(\vb{x},y),\vb{x'}) } \big[S(p(\cdot|\vb{L},(\vb{x},y),\vb{x'}),y') \nonumber \\ 
&\quad- S(p(\cdot|\vb{L},\vb{x'}),y')  \big] \Big]      \nonumber 
\end{flalign} 
where the second line is done by changing $\mathbb{E}_{ p(y'|\vb{L,x'}) } [\cdot]$ to
 $\mathbb{E}_{ p(y'|\vb{L,x,x'}) } [\cdot]$ 
 (the model is fully conditional),
 then to $\mathbb{E}_{p(y,y' | \vb{L,x,x'} ) }[\cdot]$ and rearranging.
 The second line is $\geq 0$ due to the maximum properties of scoring functions used earlier. 
 Moreover, they guarantee learning will converge to the ``truth" as follow.
\end{proof}
 
\begin{theorem}[Convergence of active learning]\label{thm-cnv}
We have a fully conditional classification model $p(y|\vb*{\theta},\vb{x})$, for $\vb*{\theta}\in \vb{\Theta}$ 
with finite discrete classes $y$ and input features $\vb{x}$.
Moreover, there is a unique ``true" model parameter $\vb*{\theta}_r$ with which the data is generated, 
the prior distribution $p(\vb*{\theta})$ satisfies $p(\vb*{\theta}_r)>0$,
and the model is identifiable.
After being applied for $n$ steps,
the AL algorithm with the acquisition functions defined above (i.e.,
$\Delta Q_{I}(\vb{x|L})$ or $\Delta Q_{S}(\vb{x|L})$) gives labeled data $\vb{L}_n$,
then $\lim_{n \rightarrow \infty} \Delta  Q_I(\vb{x|L}_n) = 0$
and likewise for $Q_{S}(\cdot)$.
Moreover, $\lim_{n \rightarrow \infty}p(\vb*{\theta}|\vb{L}_n)$
is a delta function at $\vb*{\theta}=\vb*{\theta}_r$ for data acquired by
both $\Delta Q_{I}(\vb{x|L})$ or $\Delta Q_{S}(\vb{x|L})$.
\end{theorem}

\begin{proof}
The proof for Lemma~5 in~\cite{ZhaoICLR21} can be readily adapted to show for $\vb{x}$ occurring infinitely often in 
$\vb{L}_n$ and $n\rightarrow \infty$
 $\Delta Q_{I}(\vb{x|L}_n)$ and  $\Delta Q_{S}(\vb{x|L}_n)$ both approach zero as $n\rightarrow \infty$
 since $Q(\vb{L}\cup\{(\vb{x},y)\}) \rightarrow Q(\vb{L})$ when $p(\vb*{\theta}|\vb{L}\cup\{(\vb{x},y)\})\rightarrow p(\vb*{\theta}|\vb{L})$.  Then one adapts the proof of Theorem~1 in \cite{ZhaoICLR21}, which requires finiteness and discreteness of $\vb{x}$,
 to show that $\Delta Q_{I}(\vb{x|L}_n)$ and  $\Delta Q_{S}(\vb{x|L}_n)$ both approach zero as $n\rightarrow \infty$ for all $\vb{x}$.

Now consider $\Delta Q_{I}(\vb{x|L}_n)$ which by properties of the KL function is equal to 
$\mathbb{E}_{ p(\vb*{\theta} | \vb{L}_n) }[ \mbox{KL}( p(y | \vb*{\theta},\vb{x}) || p(y | \vb{L}_n,\vb{x} ) ) ]$.
Let 
\begin{equation}
  \Theta_{NZ} = \left\{ \vb*{\theta} \,:\, \left( \lim_{n\rightarrow \infty} p(\vb*{\theta}|\vb{L}_n) \right)
   >0 \right\}
\end{equation}
Now $\vb*{\theta}_r\in \vb{\Theta}_{NZ}$ due to the arguments of Theorem~1 in \cite{ZhaoICLR21}.
From the KL approaching zero, it follows that for all $x$ and as $n\rightarrow \infty$, $ p(y|\vb*{\theta},\vb{x})$ approaches
$ p(y|\vb{L}_n,\vb{x})$ for all $\vb*{\theta}\in \vb{\Theta}_{NZ}$.
This means that all $\vb*{\theta}\in \vb{\Theta}_{NZ}$
yield identical $p(y|\vb*{\theta},\vb{x})$.
Since the model is identifiable, $\vb{\Theta}_{NZ}$ only has one element, $\vb*{\theta}_r$.

Suppose $\Delta Q_{S}(\vb{x|L}_n)\rightarrow 0$ as $n\rightarrow \infty$. 
Considering the final equation from the proof of Lemma~\ref{thm-nn}, since
 $ \mathbb{E}_{ p(y'|\vb{L}_n,(\vb{x},y),\vb{x'}) } [S(p(\cdot|\vb{L}_n,(\vb{x},y),\vb{x'}),y')- S(p(\cdot|\vb{L}_n,\vb{x'}),y') ]\geq 0$ from properties of the proper scoring rule,
 it follows that  $ \mathbb{E}_{ p(y|\vb{L}_n,\vb{x})p(y'|\vb{L}_n,(\vb{x},y),\vb{x'}) } [S(p(\cdot|\vb{L}_n,(\vb{x},y),\vb{x'}),y')- S(p(\cdot|\vb{L}_n,\vb{x'}),y') ]$ approaches 0  for all $\vb{x}$, $\vb{x'}$.
 Substitute the result from Savage about strictly proper scoring rules into the above simplification:
\begin{flalign} \label{eq:proof8}
&\mathbb{E}_{ p(y|\vb{L}_n,\vb{x})p(y'|\vb{L}_n,(\vb{x},y),\vb{x'}) } \big[S(p(\cdot|\vb{L}_n,(\vb{x},y),\vb{x'}),y')  \nonumber \\  &\quad- S(p(\cdot|\vb{L}_n,\vb{x'}),y') \big]\\
&=\mathbb{E}_{ p(y|\vb{L}_n,\vb{x})} \big[G(p(\cdot|\vb{L}_n,(\vb{x},y),\vb{x'})) \big] - G(p(\cdot|\vb{L}_n,\vb{x'})) ~, \nonumber
\end{flalign} 
which by above must mean $\rightarrow 0$ as $n \rightarrow \infty$. Because 
\begin{equation*}
\begin{aligned}
 \mathbb{E}_{ p(y|\vb{L}_n,\vb{x})} \big[p(y'|\vb{L}_n,(\vb{x},y),\vb{x'}) \big] 
 =  p(y'|\vb{L}_n,\vb{x'})   
 \end{aligned}
\end{equation*}
 and
 $G(\cdot)$ is stricly convex, it must mean that $p(y'|\vb{L}_n,(\vb{x},y),\vb{x'})$ approaches $p(y'|\vb{L}_n,\vb{x'})$   for all $y,y'$ and $\vb{x,x'}$ as $n\rightarrow \infty$.

 Now consider as $n \rightarrow \infty$,
 \wb{
\begin{equation}
\begin{aligned}
&\mathbb{E}_{p(\vb*{\theta}|\vb{L}_n)} \big[p(y|\vb*{\theta},\vb{x})^2\big] \\
&~~~= p(y|\vb{L}_n,\vb{x}) \int_\theta \frac{p(\vb*{\theta}|\vb{L}_n)  p(y|\vb*{\theta},\vb{x})}{ p(y|\vb{L}_n,\vb{x}) } p(y|\vb*{\theta},\vb{x}) \textrm{d} \theta\\
&~~~= p(y|\vb{L}_n,\vb{x})p(y|\vb{L}_n,(\vb{x},y),\vb{x}) \\ 
&~~~\rightarrow p(y|\vb{L}_n,\vb{x})^2
\end{aligned}
\end{equation}
}\noindent
and thus the variance of $p(y|\vb*{\theta},\vb{x})$ w.r.t. $p(\vb*{\theta}|\vb{L}_n)$ approaches
zero for all $\vb{x},y$.
Therefore, 
\begin{math}
 \lim_{n\rightarrow \infty}p(\vb*{\theta}|\vb{L}_n)
\end{math}
must be non-zero on a set of $\vb*{\theta}$
yielding identical 
\begin{math}
 p(y|\vb*{\theta},\vb{x})=\lim_{n\rightarrow \infty} p(y|\vb{L}_n,\vb{x}).
\end{math}
Again, there must be a unique $\theta$ with non-zero limit using identifability.
\end{proof}

Finiteness and discreteness of $\vb{x}$ is used to adapt results from \cite{ZhaoICLR21}  to show for all $\vb{x}$ that $\Delta Q(\vb{x|L}_n)\rightarrow 0$ as $n \rightarrow \infty$,
not an issue since real data is always finite.
Interestingly $\Delta Q_{I}(\vb{x|L})$, i.e., BALD,  achieves convergence too,
which occurs
because the model is identifiable and fully conditional, during AL we are free to choose $\vb{x}$ values that would distinguish different parameter values $\vb*{\theta}$.
Full conditionality also supports BEMPS
because it means any inherent bias in the AL
selection is nullified with the use of the data distribution $p(\vb{x})$.
But it also means that the theory has not been shown to hold for
semi-supervised learning algorithms, where full conditionality does not apply.

Compared with BALD, MOCU and WMOCU, 
the advantage of using strictly proper scoring rules in BEMPS is that,
\wb{as mentioned in Section~\ref{ssct-PSR},}
\ld{
the expected scores can be tailored for different inference tasks.}
\wb{BALD has its quality $Q_I(\cdot)$ as the certainty (negative entropy) of the model's parameter space.
This does not support AL when the uncertainty of some parameters does not strongly influence classification performance. }
This is reflected by its poor performance in our experiments.
MOCU however has convergence issues as ELR, as pointed out by \cite{ZhaoICLR21},
\ld{which were overcome by WMOCU and SMOCU by either augmenting or manipulating the scoring functions.}

\subsection{Scoring functions: CoreMSE and CoreLog} 
\label{sec:scoringfunctions}
Scoring rules promote the quality of predictive distributions by rewarding calibrated predictive distributions.
For example, scoring rules can be developed 
\cite{doi:10.1198/016214506000001437} for some  different inference tasks, including Brier score,
logarithmic score, the beta family, etc. 
Many loss functions used by neural networks, like cross-entropy loss, are indeed strictly proper scoring rules \cite{Lakshminarayanan2017}.
 
It is noteworthy that the acquisition function 
\begin{math}
 \Delta Q_S(\vb{x|L})
\end{math}
defined in Eq~\eqref{eq-DQS} 
is in a general form,
applicable to any strictly proper scoring function for categorical variables.
For instance, using a logarithmic scoring rule,
we have $S_{log}(p(\cdot),y) = \log p(y)$  and $G_{log}(p(\cdot)) = -I(p(\cdot))$.
The corresponding materialization of Eq~\eqref{eq-QSS} is 
\begin{align}  
 &Q_{CoreLog}(\vb{L}) \nonumber \\
 &= \e{ p(\vb{x}) p(\vb*{\theta}|\vb{L}) } {
     \kl{ p(y | \vb*{\theta}, \vb{x})} {p(y| \vb{L, x})}} 
     \label{eq:corelog}
\end{align}
Similarly, using the squared error scoring rule, known as a Brier score,
we have $S_{MSE}(p(\cdot),y)=-\sum_{\hat{y}} \left(p(\hat{y})-1_{y=\hat{y}}\right)^2 $  
and $G_{MSE}(p(\cdot)) = \sum_y p(y)^2 -1$. 
The definition of $Q(\vb{L})$ in Eq~\eqref{eq-QSS} is
\begin{align} 
& Q_{CoreMSE}(\vb{L}) \nonumber \\
& =  \e{ p(\vb{x}) p(\vb*{\theta}|\vb{L}) }{
      \sum_y \big( p(y | \vb*{\theta}, \vb{x}) - p(y| \vb{L, x})\big)^2 } 
      \label{eq:coremse}
\end{align}

\begin{algorithm}[!t]
\small
  \caption{Estimating point-wise $\Delta Q(\vb{x|L})$ with Equation~\eqref{eq-DQS} }\label{alg-qrx}
  \begin{algorithmic}[1]
    \Require unlabeled data point $\vb{x}$, existing labeled data $\vb{L}$, estimation point $\vb{x'}$ 
        \Require model/network ensemble $\vb{\Theta}=\{\vb*{\theta}_1,...,\vb*{\theta}_E\}$ built from labeled data $\vb{L}$,
    \Require strictly convex function $G(\cdot)$ taking as input a probability density over $y$ 
    \State $Q=0$
    \State $qx(\cdot) = \sum_{\vb*{\theta}\in \vb{\Theta}} p(\vb*{\theta}|\vb{L})
              p( \cdot|\vb*{\theta},\vb{x}) $
    \For{$y$}   
        \State $q(\cdot) = \sum_{\vb*{\theta}\in \vb{\Theta}} p(\vb*{\theta}|\vb{L},(\vb{x},y))
              p( \cdot|\vb*{\theta},\vb{x'}) $
        \State $Q ~+\!\!= qx(y)G(q(\cdot))$
    \EndFor
    \State $q(\cdot) = \sum_{\vb*{\theta}\in \vb{\Theta}} p(\vb*{\theta}|\vb{L})
              p( \cdot|\vb*{\theta},\vb{x'}) $    
   \State $Q~ -\!\!= G(q(\cdot))$
    \State \Return $Q$
  \end{algorithmic}
\end{algorithm}
\begin{algorithm}[!t]
  \captionof{algorithm}{Estimate of $\argmax_{\vb{x}\in \vb{U}} \Delta Q(\vb{x|L})$ }\label{alg-qr}
  \begin{algorithmic}[1]
    \Require unlabeled pool $\vb{U}$, estimation pool $\vb{X}$
   \For{$\vb{x} \in \vb{U}$}
    \State $Q_{\vb{x}}=0$
    \For{$\vb{x'} \in \vb{X}$}
         \State $Q_{\vb{x}} ~ +\!\!= \Delta Q(\vb{x|L,x'})$
    \EndFor \EndFor
    \State \Return $\argmax_{\vb{x} \in \vb{U}} Q_{\vb{x}}$ 
  \end{algorithmic}
\end{algorithm}

Viewed above, minimising Brier score gets the probability right in a least squares sense, i.e., minimising
the squared error between the predictive probability and the one-shot label representation,
which pays less attention to very low probability events.
Meanwhile, log probability gets the probability scales right, paying attention to all events. 
In most cases, we can create a particular model to match just about any Bregman divergence (e.g., minimum squared errors is a Gaussian). In practice, we can also  use robust models (e.g., a Dirichlet-multinomial rather than a multinomial, a negative binomial rather than a Poisson, Cauchy rather than Gaussian) in our log probability. 
Combining the two $G(\cdot)$ functions above with Equation~\eqref{eq-DQS} yields two acquisition functions for the different scoring rules.  
\wb{
\begin{eqnarray*}
\lefteqn{\Delta Q_{CoreMSE}(\vb{x}|\vb{L}) }&&\label{eq:QMSE}\\
&=&
\e{ p(\vb{x'}) p(y|\vb{L,x}) }  {\sum_{y'}
\left( p(y'|\vb{L},(y,\vb{x}),\vb{x'}) -
p(y'|\vb{L,x'})
\right)^2}\nonumber\\
\lefteqn{\Delta Q_{CoreLog}(\vb{x}|\vb{L}) }&&\label{eq:QLog}\\
&=&
\e{p(\vb{x'}) p(y|\vb{L,x}) }{
\kl{p(y'|\vb{L},(y,\vb{x}),\vb{x'})}{p(y'|\vb{L,x'})}
}\nonumber
\end{eqnarray*} }

\subsection{Understanding Uncertainty and Diversity}
\label{sec:UncDiv}

\wb{
How do the $\Delta Q$ functions of Eqs~\eqref{eq:QMSE} and~\eqref{eq:QLog} relate to uncertainty,
the traditional class of functions used in early AL?
The $\Delta Q_{CoreMSE}$ measures the average change in model probabilities when the label of $\vb{x}$ is acquired and
$\Delta Q_{CoreLog}$ measures the average KL between model probabilities before and after acquisition.
Thus they measure the average model change induced by 
acquiring the label of $\vb{x}$.
These are a measure of epistemic uncertainty revealed by the
data $\vb{x}$, revealing how uncertain the model is around the
data point, and do not measure aleatoric uncertainty.  Note,  $p(y|\vb{L,x})$ is a proxy for epistemic certainty in the common situation
where the true model admits high accuracy prediction, common in AL experiments.
}

\wb{
What about diversity, the other traditional function used in AL?
Eqs~\eqref{eq:QMSE} and~\eqref{eq:QLog} also hold if $\vb{x}$ is
a batch of data to be labeled.  
Moreover, they give a fundamental theoretical objective to maximise for batch AL,
whereas diversity we argue is an observed feature of good batch AL, that should be derivable from fundamental theory, such as ours.
}
\subsection{Enhanced Batch Diversity for BEMPS}
\label{ssct:algos}


Algorithm~\ref{alg-qr} gives an implementation of BEMPS for 
an arbitrary strictly convex function $G(\cdot)$, returning the data point with the best-estimated measure.
To work with a Bregman divergence or score, 
the corresponding strictly convex function $G(\cdot)$ should first be derived.
When $G(\cdot)$ is negative entropy, we call this CoreLog (Eq~\eqref{eq:corelog}) and
when $G(\cdot)$ is the  sum of squares we call this CoreMSE (Eq~\eqref{eq:coremse},
corresponding to the logarithmic or Brier scoring rules respectively.
Both Algorithms~\ref{alg-qr} and~\ref{alg-qrd} use a fixed {\it estimation pool}, $\vb{X}$, a fixed random subset of the initial unlabeled data
used to estimate expected values $\mathbb{E}_{ p(\vb{x'})}[\cdot]$.
Algorithm~\ref{alg-qr} calls Algorithm~\ref{alg-qrx} which
\ldr{implements Eq~\eqref{eq-DQS} and estimates the expected score change with respect to $\vb{x'} \in \vb{X}$, while assuming the label of $\vb{x}$.}
Note $p(\vb*{\theta}|\vb{L},(\vb{x},y))$ is computed from $p(\vb*{\theta}|\vb{L})$
via ensembles as follows:
\begin{align}
    p(\vb*\theta |\vb{L},(\vb{x},y))
        \approx \frac{
                p(\vb*\theta|\vb{L}) p(y|\vb*\theta, \vb{x})
                }{
                    \sum_{\vb*\theta\in\vb*\Theta}  
                        p(\vb*\theta|\vb{L}) p(y|\vb*\theta,\vb(x))
                }
\end{align}

Algorithm~\ref{alg-qrd} returns $B$ data points representing a batch with enhanced diversity:
it first calls Algorithm~\ref{alg-qrx} to compute, 
for each data point $\vb{x}$ in the unlabeled pool, 
a vector of expected changes in score values over the estimation pool $\vb{X}$.
Thus, this vector conveys information about  uncertainty directly
related to the change in score due to the addition of $\vb{x}$.
While the gradient embedding used in \cite{ash2019deep}
represents a data point's impact on the model, 
our vector represents a data point's direct 
impact on the mean proper score.
Concurrently Algorithm~\ref{alg-qrd} computes the estimate of $\Delta Q(\vb{x|L})$ for these same $\vb{x}$s.
The top $T$\% of scoring data $\vb{x}$ are then clustered with $k$-Means and a representative
of each cluster closest to the cluster mean is returned.
The intuition is that 1) only higher scoring data $x$ should appear in a batch;
2) those clusters capture the pattern of expected changes in score values
deduced by samples in the unlabeled pool;
3) samples in the same cluster can affect the 
learning similarly, so it should not co-occur in a batch.

\begin{algorithm}[!t] 
  \captionof{algorithm}{Finding a diverse batch}\label{alg-qrd}
  \begin{algorithmic}[1]
    \Require unlabeled pool $\vb{U}$,  batch size $B$
        \Require estimation pool $\vb{X}$,  top fraction $T$
    \State  $\forall_{\vb{x} \in \vb{U}} Q_{\vb{x}} = 0$
    \For{$\vb{x} \in \vb{U}$, $\vb{x'} \in \vb{X}$}
         \State $Q_{\vb{x}}~ +\!\!= vec_{\vb{x,x'}} = \Delta Q(\vb{x|L,x'})$
    \EndFor
    \State $V \leftarrow topk(Q,T*|\vb{U}|)$
    \State $batch = \emptyset$
    \State $centroids$ = $k$-Means centers $(vec_{\vb{x}\in V}, B)$
    \For{$c \in centroids$}
         \State $batch ~\cup\!\!= \{ \argmin_{\vb{x} \in V}||c - vec_{\vb{x}}|| \}$
    \EndFor
     \State \Return $batch$
  \end{algorithmic}
\end{algorithm}
\begin{algorithm}[!t] 
\small
  \caption{Bayesian Estimate of Mean Proper Scores with Ensembles}\label{alg-ensemble}
  \begin{algorithmic}[1] 
    \Require initial unlabeled data $\vb{U}$, initial labeled data $\vb{L}$, 
    \Require batch size $B$, estimation pool $\vb{X}$, the number of acquisition iteration $N$
    
    \State Initialize: $n=0,\vb{L}_0\leftarrow \vb{L}, \vb{U}_0\leftarrow \vb{U}$
    \While {$n<N$} 
        \If{Ensemble method is Deep Ensembles}
        \ldr{
        \While {$e<E$} 
                \State Randomly split $\vb{L}_i$ into training and validation
                \State  Train a model $\vb*{\theta}_{i,e}$ 
        \EndWhile
        \EndIf
        \If{Ensemble method is MC-Dropout}
            \State Randomly split $\vb{L}_n$ into training and validation
            \State Train a single model $\vb*{\theta}_n$ on the split
            \While {$e<E$} 
                \State Compute $\vb*{\theta}_{n,e}$ via a single forward pass through \\
                \hspace*{4.2em} the current trained model 
                $\vb*{\theta}_n$ with MC-Dropout
            \EndWhile
        \EndIf
        \State $\vb*\Theta_{n} = \{\vb*{\theta}_{n,1}, \vb*{\theta}_{n,2},..., \vb*{\theta}_{n,E}\}$
        \State $batch = \emptyset $
        \If{$B$ is one}
        \State $batch = \argmax_{\vb{x}\in \vb{U}} \Delta Q(\vb{x|L})$ implemented by\newline
        \hspace*{2.7em} Algorithm~\ref{alg-qr} 
        \Else{ 
        \State $batch = B$ samples acquired with 
        Algorithm~\ref{alg-qrd}}
         \State $\vb{L}_{n+1}\leftarrow \vb{L}_n \cup batch$
        \State $\vb{U}_{n+1}\leftarrow \vb{U}_n \setminus batch$
        \EndIf}
    \EndWhile
  \end{algorithmic}
\end{algorithm}


\subsection{Ensembles for BEMPS}
\label{sec:emsemlesBEMPS}

\ldr{
Computing $p(\cdot | \vb{L}, x)$ and $p(\cdot | \vb{L}, (\vb{x}, y), \vb{x}')$ 
is intractable.
We approximate the integral over $\vb*\theta$ via }
two ensemble methods in Algorithm~\ref{alg-ensemble}:  Deep Ensembles and MC-Dropout. 
We first implement Deep Ensembles since the method has been successfully used to improve predictive performance \cite{Lakshminarayanan2017},
\ldr{
MC-Dropout, trying to
peed up the acquisition process}

\textbf{Deep Ensembles.}
\ldr{
Let $\vb{\Theta}$ denote a set of DNN models 
(i.e., DistilBERT \cite{sanh2019distilbert} 
for text, and VGG-16 \cite{Simonyan2014VeryDC} for images in our experiments)
with size $E$.
Each individual model, denoted as $\vb*\theta_e$,
is trained with a randomly generated training-validation split of
incrementally augmented labeled pool $\vb{L}$
at each acquisition iteration.
The training-validation ratio is set to 70/30.
In other words,
different models are trained utilizing different
train/validation splits, 
which fosters ensemble diversification.}
We call this split process dynamic validation
set (aka Dynamic VS). Once the ensemble models are trained, a batch of samples can be acquired via Algorithm~\ref{alg-qrd}.
This dynamic approach to validation was found to increase the training efficiency and model’s performance,
as demonstrated in our ablation studies.


\textbf{MC-Dropout.}
One drawback of using Deep Ensembles with DNNs 
is its high computational cost in terms of
\ld{training the active learner, i.e., the classifier.}
\cite{gal2016dropout} demonstrated that Monte Carlo samples of the posterior can be obtained by conducting several stochastic forward passes at test time. 
\ld{Thus, an ensemble of $E$ models often used in standard deep ensemble methods can be replaced
 with MC-Dropout models in the AL process to approximate the integral over $\vb*\theta$  \cite{pop2018deep}.}
In other words,
the predictive distribution $p(\cdot|\vb{L,x})$ and $p(\cdot|\vb{L},(\vb{x},y),\vb{x'})$ could be estimated through $E$ forward passes on a singled trained model, 
\ld{
rather than training the classifier $E$ times, and thus saving computational cost, particularly when the value of $E$ becomes large. Refer to Table~\ref{tab:table5ensemblestime}.}

\section{Experiments}
\label{sec:Experiments}

We conducted comprehensive sets of experiments on various 
classification tasks to demonstrate
the efficacy of BEMPS by comparing CoreMSE and CoreLog
to some recent AL methods. 
To show the reliability and robustness of our BEMPS framework,
we consider
four benchmark text datasets and two benchmark image datasets
for either binary classification or multi-class classification tasks.
Moreover, 
going beyond the quantitative studies, 
we further visualise different AL methods using Data Maps \cite{swayamdipta2020dataset} 
and t-SNE to 
have an in-depth understanding of their sampling behaviors along with model calibration.

\begin{table}[!t]
  \caption{Four benchmark text datasets and the used language model}
  \label{tab:table1}
  \centering\small
  \begin{tabular}{llll}
    \toprule
    Dataset & \makecell{Unlabeled/\\Test sizes} & \makecell{Lang.\\~Model} & \makecell{Initial \\ labeled size}  \\
    \midrule
    AG NEWS  & 120,000 / 7,600 & DistilBERT  &   26  \\
    PUBMED & 15,000 / 2,500  & DistilBERT & 26    \\
    IMDB     &  25,000 / 25,000  & DistilBERT       & 26 \\
    SST5     & 8544 / 2210 & DistilBERT       & 26   \\
    \bottomrule
  \end{tabular}

\nodialogue{
  \caption{Educational dialogue act datasets and the used language model}
  \label{tab:table2}
  \centering\small
  \begin{tabular}{llll}
    \toprule
      Dataset & \makecell{Unlabeled/\\Test sizes} & \makecell{Lang.\\~Model} & \makecell{Initial \\ labeled size}  \\
    \midrule
    First-level    & 3763 / 476 & DistilBERT       & 50   \\
    Second-level     & 3763 / 476 & DistilBERT       & 50   \\
    \bottomrule
  \end{tabular}
  }
\vspace{5pt}

  \caption{Image datasets and the used models}
  \label{tab:table3}
  \centering\small
  \begin{tabular}{llll}
    \toprule
    Dataset & \makecell{Unlabeled/\\Test sizes} & \makecell{NN\\~Model} & \makecell{Initial \\ labeled size}  \\
    
    \midrule
    MNIST  & 60,000 / 10,000 & BNN  &   20  \\
    CIFAR10 & 50,000 / 10,000  & VGG-16bn & 20    \\
    \bottomrule
  \end{tabular}
\end{table}

\subsection{Text classification}

\begin{figure*}[!t]
\centering
    \includegraphics[width=0.8\textwidth]{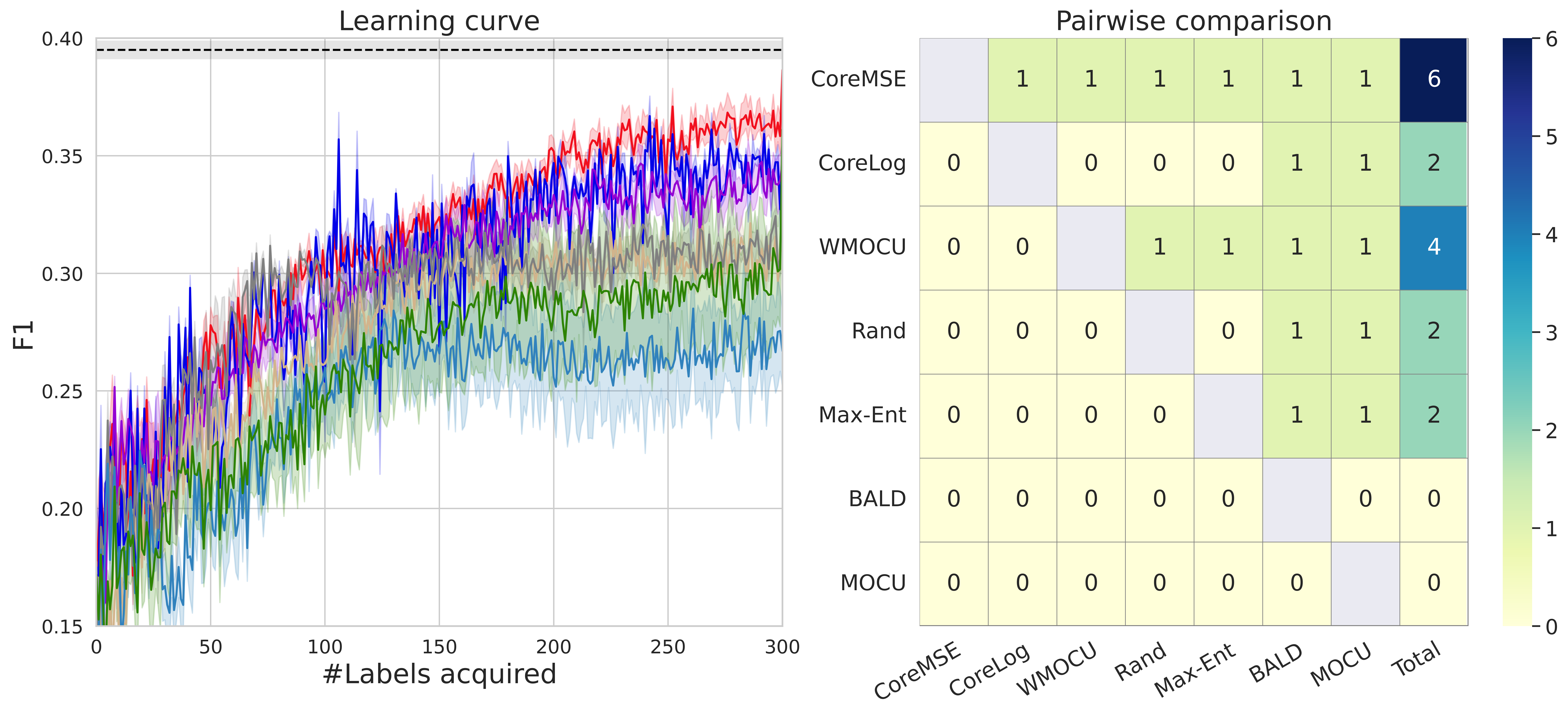}
    \includegraphics[width=0.9\textwidth]{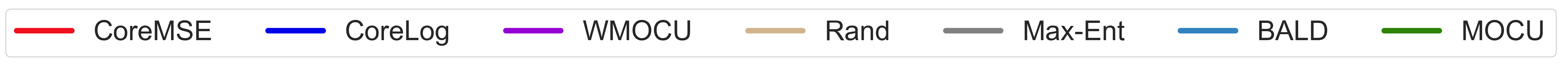}
     \caption{Performance on SST5 dataset. The left half illustrates the learning curve, while the right half illustrates the matrix of paired comparisons. 
     The dashline represents the performance of the backbone classifier trained on the entire dataset.
    }
    \label{fig:uncertaintybasedAL_performance}
\end{figure*}



\subsubsection{Datasets} We use {\it four benchmark
datasets}: IMBD, AG NEWS, PUBMED and SST5 for text classification as shown in \autoref{tab:table1}.
\wtr{IMDB contains 50K movie reviews, categorized as either Positive or Negative \cite{maas2011learning}.
AG NEWS includes 120K texts distributed across four classes: Science/Technology, World, Business, and Sports \cite{zhang2015character}. \cite{zhang2015character}. PUBMED 20k is designed for sentence classification and contains approximately 20K medical abstracts divided into five classes: Objective, Background, Conclusions, Results, and Methods \cite{dernoncourt2017pubmed}. 
Similarly, SST5 includes 11K sentences from movie reviews, labeled with five sentiment classes \cite{socher2013recursive}.}

\subsubsection{Baselines} \label{sssec:sectxbaselines} 
\ldr{Let $B$ indicate the batch side.}
We consider two types of performance comparisons: non-batch active learning \ldr{($B = 1$)} and batch active learning \ldr{($B > 1$)}. 
\wtr{In the case of non-batch active learning, CoreMSE and CoreLog are compared against Max-Entropy \cite{yang2016active}, BALD \cite{Houlsby2011}, MOCU \cite{ZhaoICLR21}, and WMOCU \cite{ZhaoICLR21}, along with a random baseline. For batch active learning, CoreMSE and CoreLog are compared to WMOCU, BADGE \cite{ash2019deep}, and ALPS \cite{yuan-etal-2020-cold}, also alongside a random baseline.}

\textbf{CoreMSE \& CoreLog.} In the default configuration of BEMPS we used deep ensembles. 
For non-batch active learning, we chose a sample with the highest uncertainty score related to classification error. 
For batch active learning, each sample in the unlabeled pool is represented as a vector of scores computed by \ldr{Eq~\eqref{eq-DQS}}. 
We then used $k$-MEANS to generate $B$ clusters and select from each cluster the sample closest to the cluster center to form the batch.

\textbf{Random.} We sampled $B$ samples uniformly from the unlabeled pool.

\textbf{Max-Ent.} We chose a sample with the highest entropy of the predictive distribution \cite{lewis1994sequential, wang2014new}. 

\textbf{BALD.} Like Max-Entropy, we chose a sample with the maximum mutual information based on how well labeling those samples would improve the model parameters \cite{Houlsby2011}. 

\textbf{MOCU \& WMOCU.} Following \cite{ZhaoICLR21}, for non-batch active learning, we selected a sample with the highest uncertainty score related to classification error. Similarly to CoreMSE and CoreLog, we represent each sample in the unlabeled pool as a vector of scores, as computed by Eqs~(5)and (10) in \cite{ZhaoICLR21}. 
For batch active learning, we selected $B$ samples using the k-MEANS clustering approach.

\textbf{BADGE.} Following \cite{ash2019deep}, we represented each sample with a gradient embedding generated from a pretrained language model, specifically DistilBERT in our experiments. Then, using $k$-MEANS++, we generate $B$ clusters and choose a representative from each cluster that is closest to the mean.

\textbf{ALPS.} Different from BADGE, we followed \cite{ash2019deep} to generate surprisal  embeddings from DistilBERT as
inputs to $k$-MEANS. 
Then, a similar approach was used to choose the $B$ samples in a batch.

\subsubsection{Model configuration}
We used a small and fast pretrained language model,
DistilBERT \cite{sanh2019distilbert} as the backbone classifier 
in our experiments.
We fine-tuned DistilBERT on each dataset after each AL iteration with a random re-initialization \cite{frankle2018lottery},
 proven to improve the model performance over the use of incremental fine-tuning with the newly acquired samples \cite{gal2017deep}.
The maximum sequence length was set to 128, 
and a maximum of 30 epochs was used in fine-tuning DistilBERT with early stopping \cite{dodge2020fine}.
We used AdamW \cite{loshchilov2017decoupled} as the optimizer with a learning rate 2e-5 and betas 0.9/0.999. All experiments were run on 8 Tesla 16GB V100 GPUs.

Each AL method was run for 5 times with different random seeds on each
dataset. The batch size $b$ was set to 1 for the non-batch active learning,
and for batch active learning, the batch size $B$ 
is set to \{5, 10, 50, 100\}.
We trained five DistilBERTs with Dynamic VS \ldr{as members of a deep ensemble,}
as described in Algorithm~\ref{alg-ensemble}.
Note that our implementation of CoreMSE and CoreLog does not rely on a separate large validation set.

\begin{figure*}[!t]
\centering
    \includegraphics[width=0.8\textwidth]{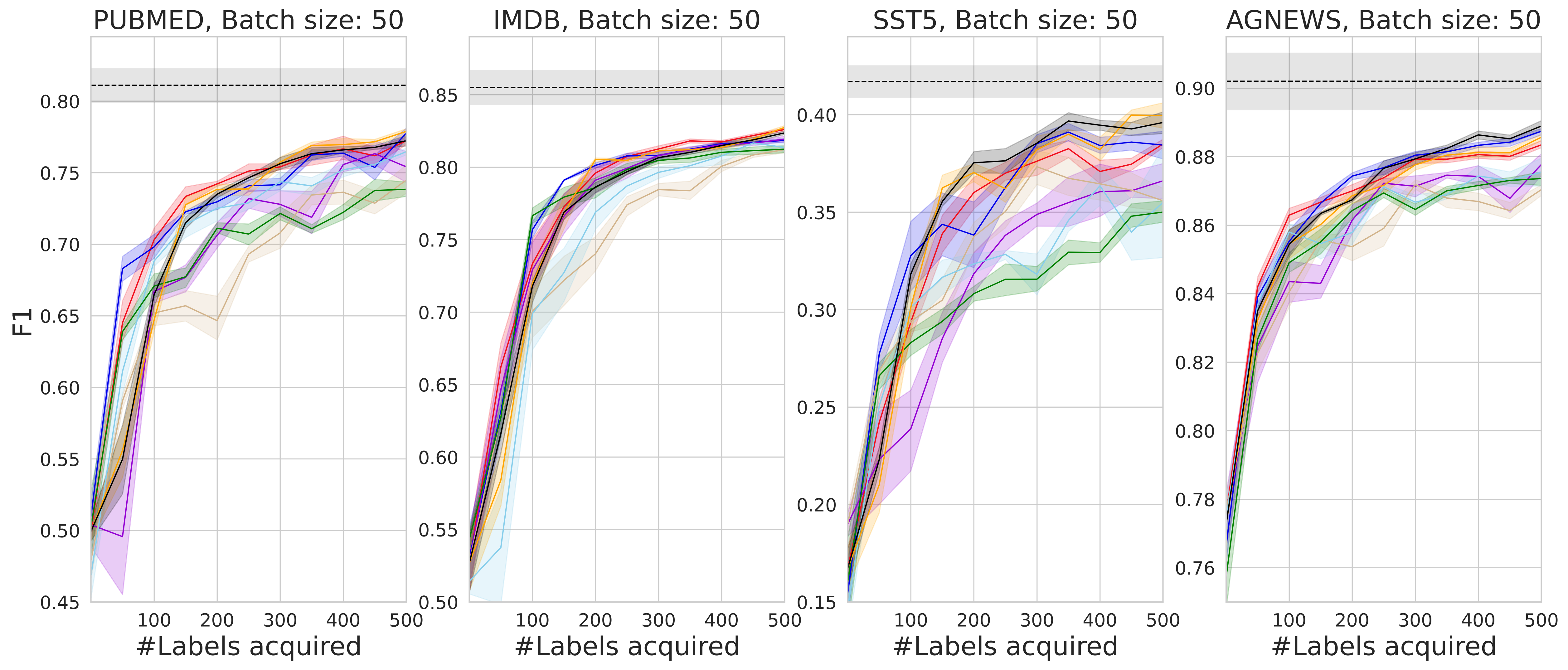}\\
    \includegraphics[width=0.9\textwidth]{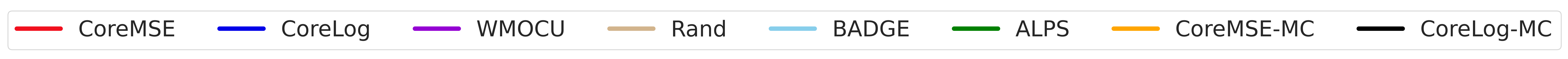}
    \caption{Learning curves of batch size 50 for PUBMED, IMDB, SST5 and AG NEWS. The dashline represents the performance of the backbone classifier trained on the entire dataset.}
    \label{fig:b50_f1_lc}
\end{figure*}
\begin{figure*}[!t]
     \centering
    \begin{subfigure}[t]{0.45\textwidth}
        \centering
        \includegraphics[width=\textwidth]{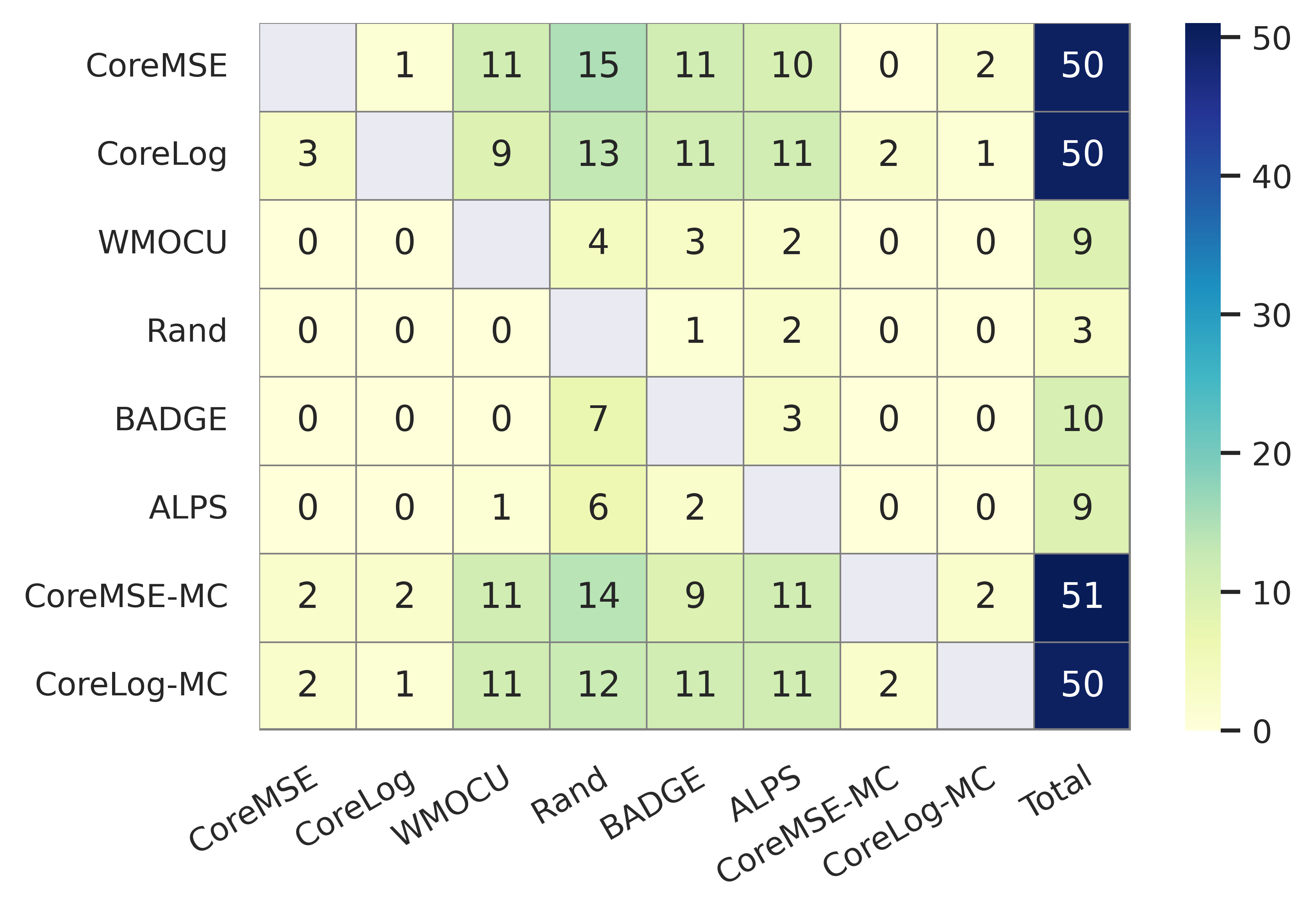}
        \caption{F1-based pairwise comparison}
    \end{subfigure}%
    ~ 
    \begin{subfigure}[t]{0.45\textwidth}
        \centering
        \includegraphics[width=\textwidth]{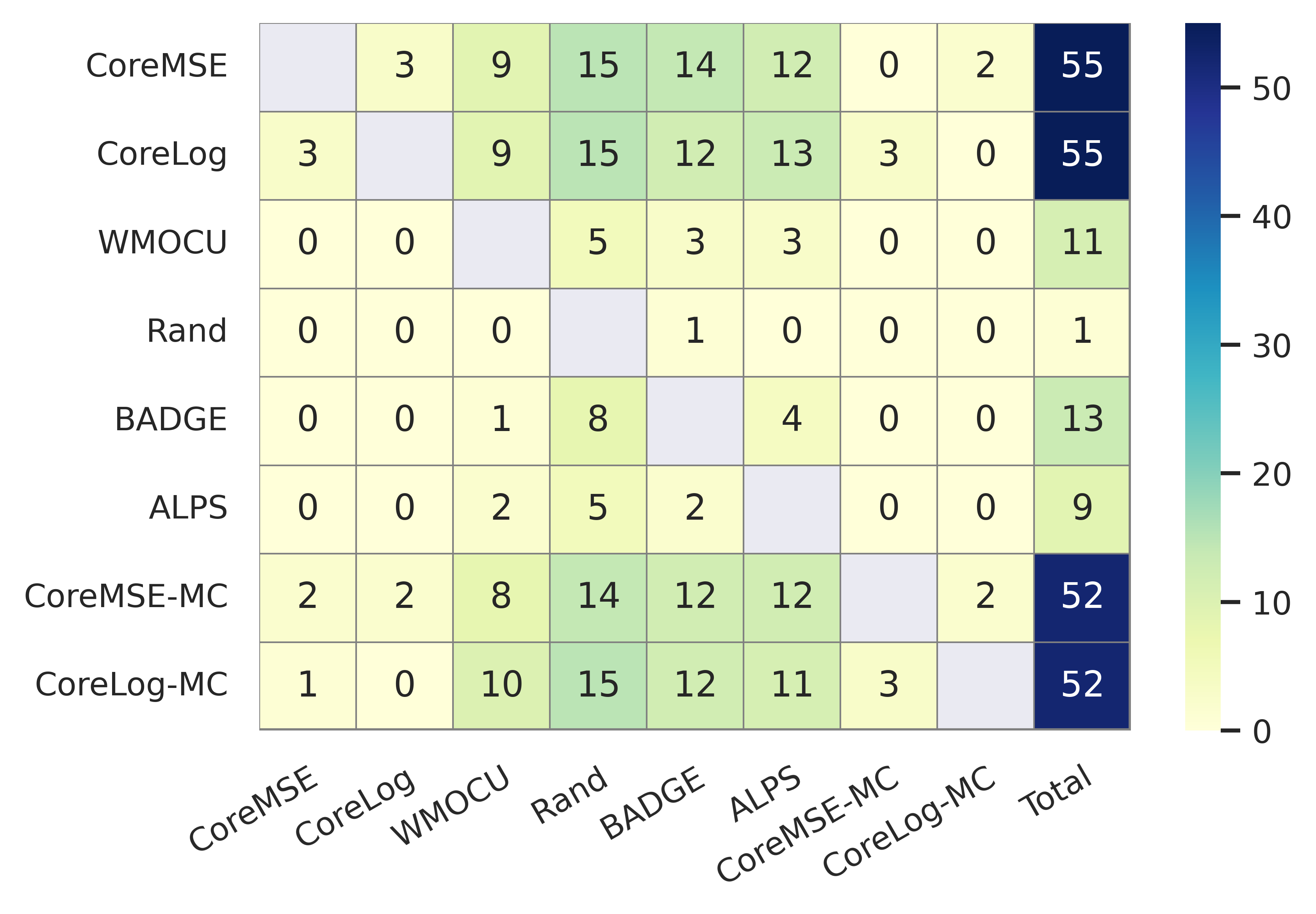}
        \caption{Accuracy-based pairwise comparison}
    \end{subfigure}
    \caption{Pairwise comparison matrices of batch active learning strategies on four datasets with four batch sizes. }
    \label{fig:divercity_sum_pc}
\end{figure*}

\subsubsection{Comparative performance metrics}   \label{sssec: modelconfiguration}
We followed \cite{ash2019deep} to compute a pairwise comparison matrix
but used a counting-based algorithm \cite{shah2017simple},
as shown in the right of Fig.~\ref{fig:uncertaintybasedAL_performance}.
The rows and columns of the matrix 
correspond to the AL methods used in our experiments. 
Each cell represents the outcome of the comparison between method $i$ and method $j$ over all datasets (D). 
Let $C_{i,j,d} = 1$ when method $i$ beats method $j$ on dataset $d$, 
and 0 otherwise.
The value of each cell is computed as
$C_{i,j}=\sum_d^DC_{i,j, d}$.
To determine the value of each $C_{i,j,d}$ entry, 
we used a two-side paired $t$-test to compare their performance for five weighted 
F1 scores (or accuracy) at maximally spaced labeled sample sizes $\{l_{i,j,d}^1, l_{i,j,d}^2, ...,l_{i,j,d}^5\}$ 
from the learning curve. 
We computed the $t$-score as $t = \sqrt{5}\hat{\mu}/\hat{\sigma}$, where
$\hat{\mu}$ and $\hat{\sigma}$ are the usual sample mean and standard deviation.
For example, in Fig.~\ref{fig:uncertaintybasedAL_performance}, five samples according to a step size 50 were used: the first sample $l_{i,j,d}^1$ was chosen after $50$ labels, the second $l_{i,j,d}^1$ after $100$, etc.
Whereas, for Fig.~\ref{fig:b50_f1_lc}, a step size 100 was used.
$\hat{\mu}=\frac{1}{5}\sum_{k=1}^5\left(l_{i,j,d}^k\right), \hat{\sigma} = \sqrt{\frac{1}{4}\sum_{k=1}^5\left(l_{i,j,d}^k - \hat{\mu}^2\right)}$. 
The value $C_{i,j,d}$ is set to 1 if method $i$ outperforms method $j$ with a $t$-score greater than 2.776 (corresponding to a $p$-value less than 0.05).
We summed the outcomes of each pairwise comparison to determine the total performance score for each method, as displayed in the 'Total' column of the matrix.
The method with the highest total score is ranked as the best among the AL methods evaluated.

\subsubsection{Results and Analysis}

\textbf{Non-batch active learning.} 
We first compared our CoreMSE and CoreLog based on Algorithm~\ref{alg-qr} to the baselines on the SST5 datasets
to demonstrate how those methods perform particularly in a hard classification setting
where classes are imbalanced.
The learning curve sitting on the left of Fig.~\ref{fig:uncertaintybasedAL_performance} shows
CoreMSE, CoreLog and WMOCU outperform all the other methods considered,  
we can attribute this to 
their estimation of uncertainty being better related to classification accuracy.
Among these three methods, our CoreMSE performs the best in terms of F1 score.
The matrix at the right of
Fig.~\ref{fig:uncertaintybasedAL_performance} then presents a statistical summary of comparative performance.
CoreMSE has the highest total quantity which further confirms its effectiveness in acquiring informative samples in AL. 


\textbf{Batch active learning.}
We compared batch CoreMSE and CoreLog implemented based on Algorithm~\ref{alg-ensemble} with
BADGE, ALPS 
and batch WMOCU on the four datasets listed
in \autoref{tab:table1},
We extended WMOCU with our Algorithm~\ref{alg-ensemble} to build its batch counterpart.
Specifically, we generated $vec_{\vb{x}}$ using its point-wise error estimates, i.e., 
Eq~(10) in \cite{ZhaoICLR21}. 
The random baseline selects $B$ unlabeled samples randomly.
Here we present the results derived with $B=50$ as an example.
More comprehensive results with different batch sizes, including accuracy,
can be found in Appendix~B.

The learning curves in Fig.~\ref{fig:b50_f1_lc} show that batch CoreMSE and CoreLog almost always outperform the other AL methods as
the number of acquired samples increases. 
Batch WMOCU devised with our batch algorithm compare favourably 
with BADGE and ALPS that use gradient/surprisal embeddings to increase batch diversity.
These results suggest that selecting the representative samples from clusters learned with vectors
of expected changes in scores (i.e., Eq~\eqref{eq-DQS})  is appropriate, leading to an improved AL performance.
Moreover, the performance differences between our methods and others 
on PUBMED and SST-5 indicate that
batch CoreMSE and CoreLog can still achieve good results 
when the annotation budget is limited in those imbalanced datasets.

We also created two pairwise comparison matrices for different batch sizes using either F1 score or accuracy.
Fig.~\ref{fig:divercity_sum_pc} show the sum of the four matrices, summarizing 
the comparative performance on the four datasets.
The maximum cell value is now $4\times 4 = 16$. 
In other words,
if a method beats another on all the four datasets across the four different batch sizes,
the corresponding cell value will be 16.
Both matrices computed with F1 score and accuracy respectively 
show both CoreMSE and CoreLog are ranked higher than the other methods. 
Thus, we attribute the better performance of CoreMSE and CoreLog to the fact that their estimation of uncertainty is more correlated to prediction performance than other methods studied.

\subsection{Ablation Studies}
\label{sec:ablation}

\begin{figure*}[!t]
\centering
    \begin{subfigure}[t]{0.4\textwidth}
        \centering
        \includegraphics[width=1\textwidth]{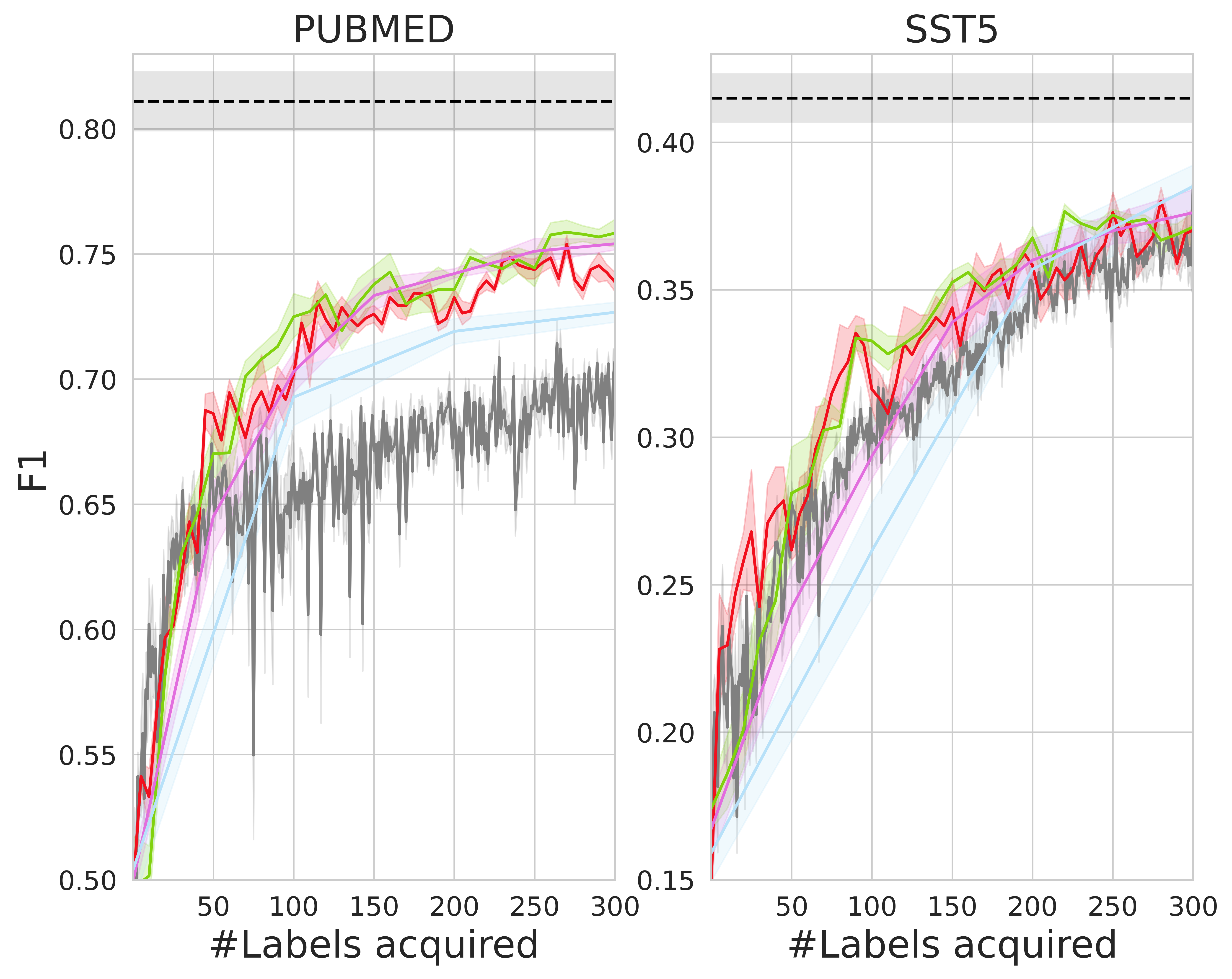}
       \includegraphics[width=1\textwidth]{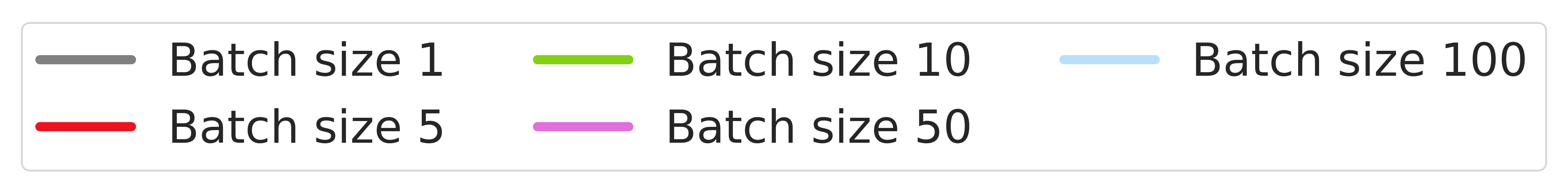}
        \caption{Learning curves of various batch size for CoreMSE.}
          \label{fig:b50_batchsize}
    \end{subfigure}%
    ~ 
    \begin{subfigure}[t]{0.4\textwidth}
        \centering
        \includegraphics[width=1\textwidth]{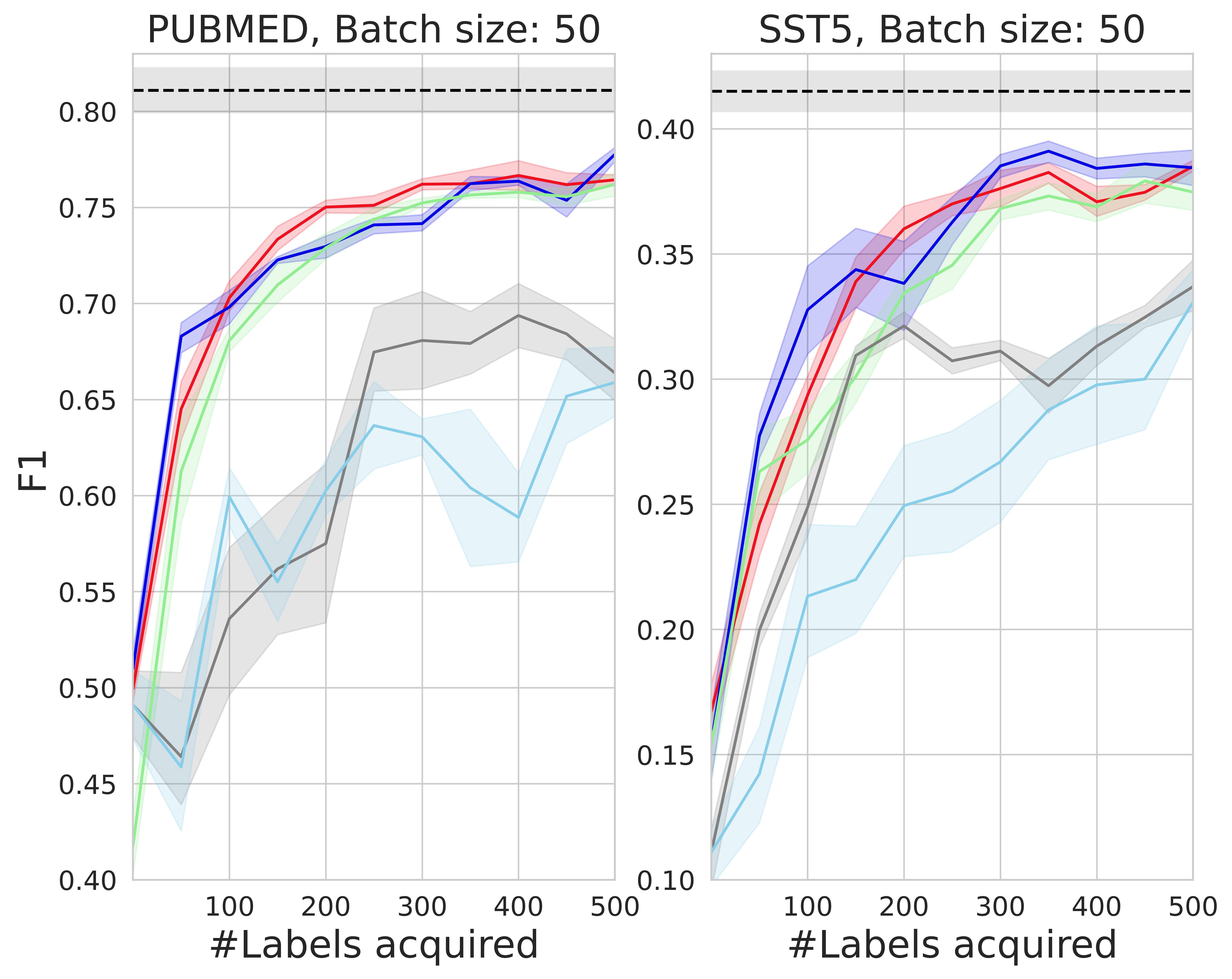}
        \includegraphics[width=1\textwidth]{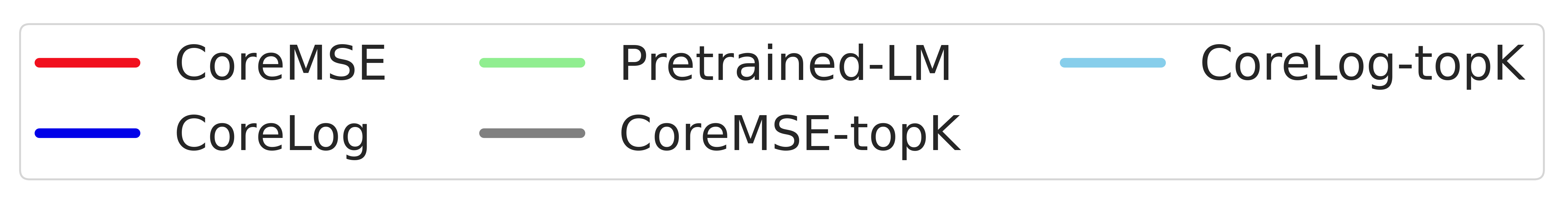}
        \caption{Learning curves of the model training with diversity.}
          \label{fig:b50_diversity}
    \end{subfigure}
    \vspace{-2mm}
    \caption{Learning curves for batch size and batch diversity on PUBMED and SST5. The dashline represents the performance of the backbone classifier trained on the entire dataset.
    }
    \label{fig:b50_batchsize_diversity}
\end{figure*}

\textbf{Batch size.} 
 In  Fig.~\ref{fig:b50_batchsize}, we plotted the learning curves of batch CoreMSE with various batch sizes (i.e., $B \in \{1, 5, 10, 50, 100\}$) for PUBMED and SST5. 
The curves demonstrate that the performance of smaller batch sizes (5 or 10) is better to that of larger batch sizes (50 or 100), 
particularly in the early training iterations. 
Moveover, 
the batch version of CoreMSE and CoreLog perform substantially better than their non-batch counterparts.
In order to acquire $B$ samples, 
the non-batch (i.e., $B=1$) case must perform multiple one-step-look-ahead optimising acquisitions consecutively (Algorithm~\ref{alg-qrx} and Algorithm~\ref{alg-qr}). 
In contrast, the batch case must only perform the acquisition once, 
heuristically, introducing a form of diversity based directly on the error surface (Algorithm~\ref{alg-qrx} and Algorithm~\ref{alg-qrd}).
While the importance of diversity is illustrated in \cite{zhou2020understanding},
they used larger batch sizes in their experiments.
a phenomena also seen with BatchBALD \cite{KAJvAYG2019}.  
Thus \wbr{we hypothesise} that the one-step-look-ahead of
Equation~\eqref{eq-dQ} is only greedy, and not optimal.
\wb{
In other words,
diversity of the batch of data is an advantage for the $\Delta Q$ functions of Equations~\eqref{eq:QMSE} and~\eqref{eq:QLog}, and this occurs mainly for smaller batch sizes where the projections for model probability estimates after acquiring labeled data retain some accuracy.}
Overall, the learning curves show that batch algorithms outperform non-batch algorithms, 
and that smaller batch sizes are preferable to large batch sizes.




\textbf{Batch diversity.}
To further study the effectiveness of Algorithm~\ref{alg-qrd},
we considered the following variants: 
1) Pretrained-LM: 
We used the embedding generated by 
DistilBERT for $k$-Means clustering, 
which is similar to the BERT-KM in \cite{yuan-etal-2020-cold};
and 2) CoreMSE-topK and CoreLog-topK: We simply chose the top $B$ samples ranked by $Q_x$. 
the results shown in Fig.~\ref{fig:b50_diversity} show that 
batch CoreMSE and CoreLog perform much better than the corresponding
CoreMSE-topK and CoreLog-topK, 
which showcases Algorithm~\ref{alg-qrd} can promote batch diversity
that benefits AL for text classification.
The performance difference between Pretrained-LM and batch CoreMSE/CoreLog
indicates that representing each unlabeled sample as a vector of expected changes in scores (i.e., Eq~\eqref{eq-DQS}) is
effective in capturing the information to be used for diversification.

\begin{figure*}[!t]
\centering
        \includegraphics[width=0.8\textwidth]{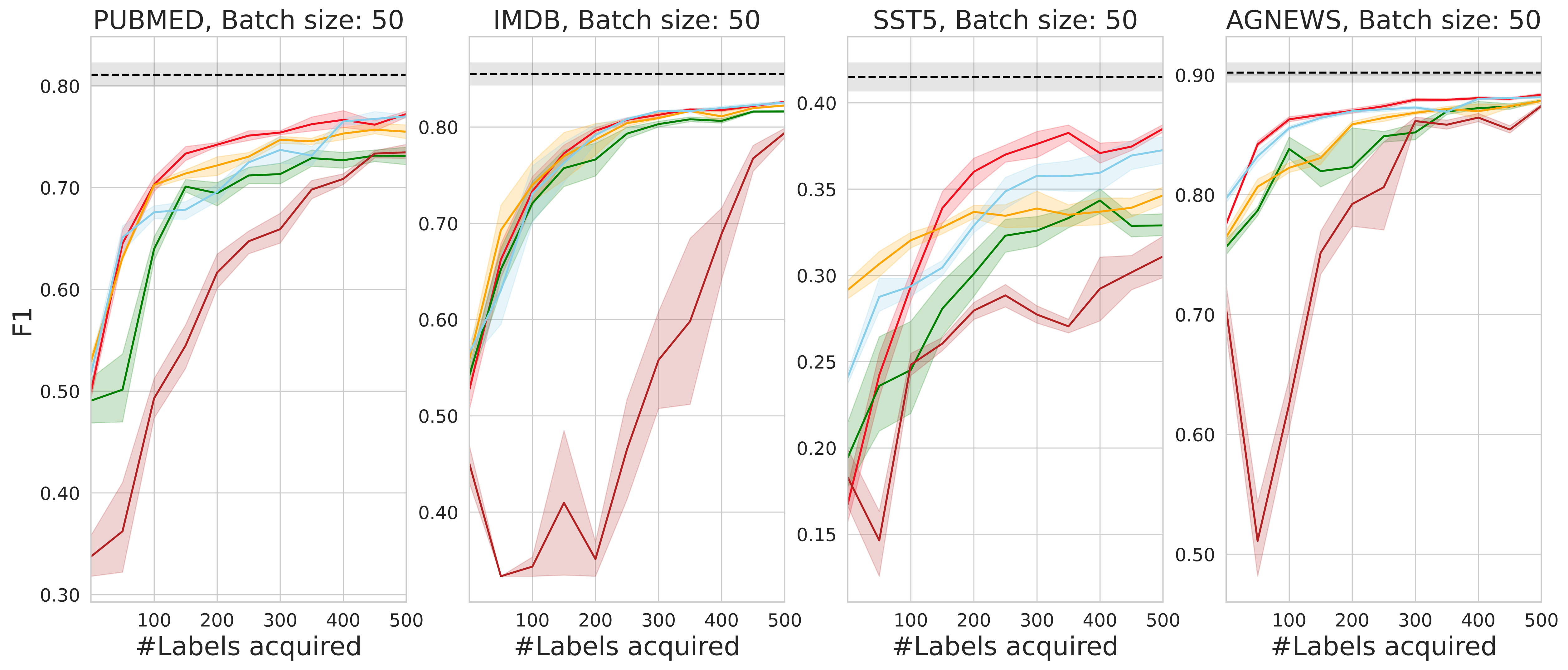}
    \includegraphics[width=0.9\textwidth]{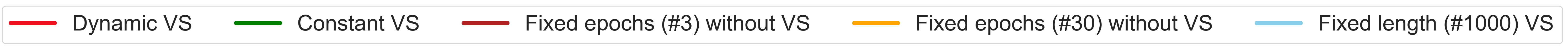}
    \caption{Learning curves of the model training with a dynamic validation set, constant validation set, fixed \# epochs without validation set, fixed length \# labels validation set for CoreMSE. The dashline represents the performance of the backbone classifier trained on the entire dataset.}
    \label{fig:b50_mixture_vs_coremse}
\end{figure*}


\textbf{Dynamic VS.}
We studies how Dynamic VS impacts the ensemble model training by comparing 
batch CoreMSE with Dynamic VS to
its following variations:
1) \textit{3/30 epochs without VS}: ensemble model training without VS and each model is trained for 3 or 30 epochs \cite{dodge2020fine},
2) \textit{Fixed length (\#1000) VS}: a pre-fixed validation set with 1000 labeled samples separate from the labeled pool, used in some existing AL empirical work;
3) \textit{Constant VS}: a variant of Dynamic VS where one random split is generated after each AL iteration and then shared by all the ensemble models.
Fig.~\ref{fig:b50_mixture_vs_coremse} shows 
Dynamic VS gains an advantage after the third acquisition iteration
on PUBMED and SST5.
It is not surprising that 30 epochs without VS and Fixed length VS perform
better in the early acquisition iterations, since they use the whole augmented labeled pool in training DistilBERT, 
whereas CoreMSE used 70\%.
But choosing the number of epochs without a validation set is simply heuristic
otherwise.
Also Fixed length VS is midway between Constant VS and Dynamic VS, indicating the variability in ensembles inherent in the dynamic training sets is a source of improvement.


\textbf{Deep Ensemble vs MC-Dropout}.
CoreMSE and CoreLog use deep ensembles,
and when MC-Dropout is used instead, they are denoted as {\bf CoreMSE-MC} and {\bf CoreLog-MC}.
Fig.~\ref{fig:b50_f1_lc} 
also shows that the performance of CoreMSE-MC and CoreLog-MC with MC-dropout are quite similar to CoreMSE and CoreLog with deep ensembles with the same annotation budget (i.e., 500) and the same batch size (i.e., 50). 
The F1 and accuracy rankings not only closely match between CoreMSE-MC and CoreMSE in deep ensembles but also exhibit a similar trend between CoreLog-MC and CoreLog, as illustrated in Fig.~\ref{fig:divercity_sum_pc}.
This result is unexpected due to the literature on ensembles \cite{durasov2021masksembles}.
\wt{However, 
MC-Dropout achieves similar performance to Deep Ensembles with substantially less training time, 
\ld{as shown
in both \autoref{tab:table5time} and Table~\autoref{tab:table5ensemblestime}.
The training speedup it offers ranges from a 4-fold increase to a 7-fold increase, 
depending on the size of the ensemble.}
\wb{and details on the interaction with ensembling can be found in \autoref{tab:table5ensemblestime}.}}

\nodialogue{
\subsection{Dialogue act classification}
To understand the performance of deep ensembles and MC-Dropout for CoreMSE and CoreLog on the highly imbalanced datasets,
we use the two real-world datasets from educational tutorial dialogues\cite{lin2022good}: 
First-level Tag and Second-level Tag, as shown in \autoref{tab:table2}. 
First-level Tag contains 4K samples of 12 imbalanced classes. Second-level Tag is highly imbalanced dataset which contains 4K samples of 31 imbalanced classes. 
Similar to \autoref{sssec:sectxbaselines}, we consider WMOCU \cite{ZhaoICLR21}, BADGE \cite{ash2019deep}, ALPS \cite{yuan-etal-2020-cold} and random as the four baselines for the batch active learning. 
We deploy the same model DistilBERT as the backbone classifier for classification tasks. 
The sample size of the initial labeled pool is set at 50 for model training.
A batch size of 50 is specified. 
We set the maximum sequence length to 128 and fine-tuned for 30 epoches. 
AdamW optimizer was used with the learning rate of 2e-5 to optimize the training of the classifier. 
All experiments were implemented on RTX 3090 and Intel Core i9 CPU processor. 
We compute pairwise comparison matrix for ranking AL methods according to the method listed in  \autoref{sssec: modelconfiguration}. 
More detailed experimental settings are given in Appendix~B

The results in Fig.~\ref{fig:figure6_lc_sum_pc} demonstrate how those methods perform particularly on a hard classification setting
where classes are highly imbalanced.
Fig.~\ref{fig:figure6_lc_sum_pc} shows both the learning curves and the F1 score pairwise comparison with a batch size 50.
Based on the F1 scoring matrices, both CoreMSE and CoreLog with MC-Dropout have better performance than the deep ensembles CoreMSE and CoreLog. Significantly, CoreLog-MC performs the best among all other methods. This Phenomenon is quite similar to the trend of other imbalance datasets such as SST5 and PUBMED in Fig.~\ref{fig:b50_f1_lc}. 
Thus, we observe that the MC-Dropout active learning effectively compensates for this imbalance during acquisition,
especially for the highly imbalanced datasets \cite{beluch2018power, chen2021imbalanced, li2020imbalance}. 
}


\begin{figure*}[t]
\centering
    \begin{subfigure}[t]{0.45\textwidth}
        \centering
        \includegraphics[width=1\textwidth]{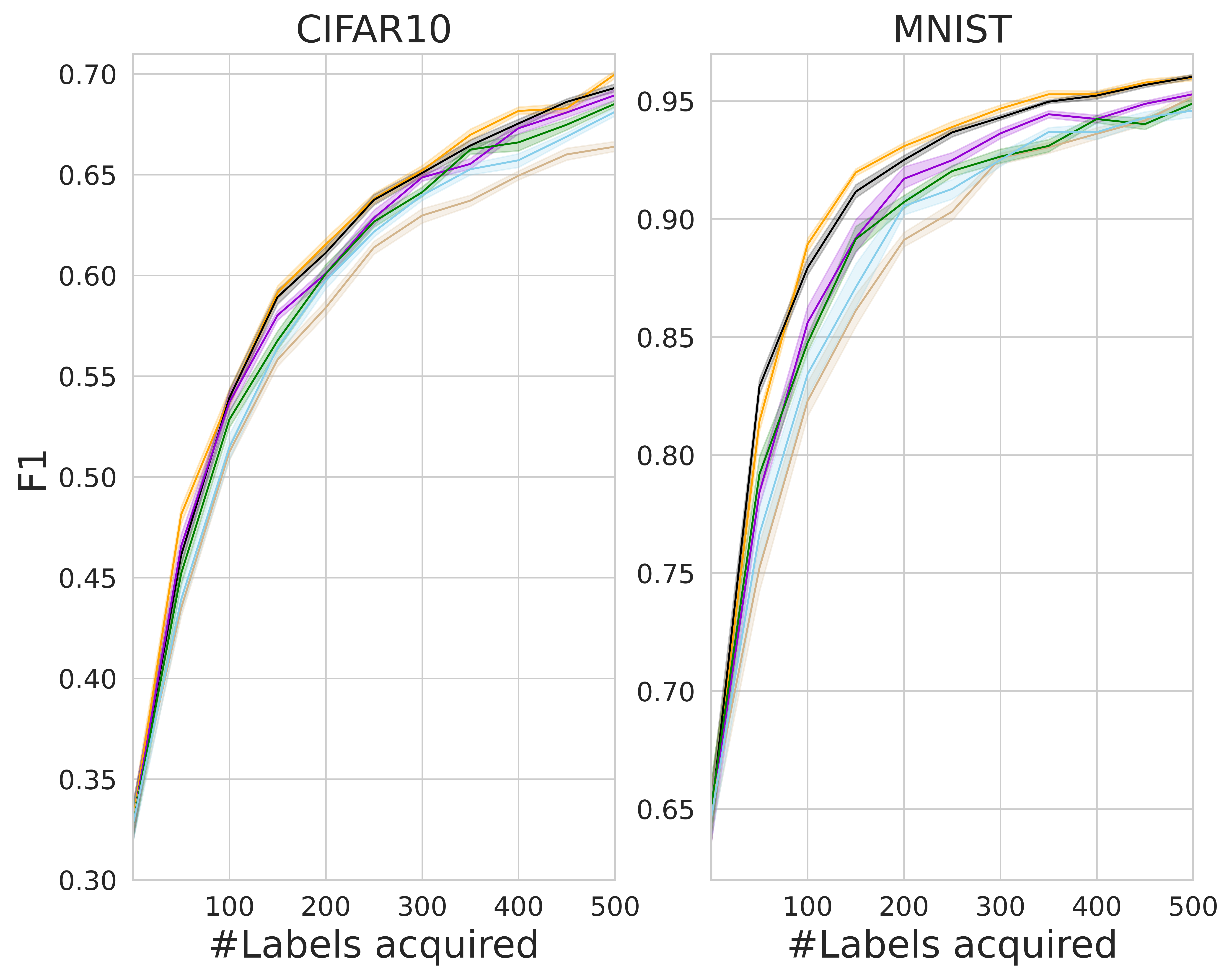}
        \caption{Learning curves of batch size 50}
          \label{fig:figure8_lc_sum_pc_a}
    \end{subfigure}%
    ~ 
    \begin{subfigure}[t]{0.45\textwidth}
        \centering
        \includegraphics[width=1\textwidth]{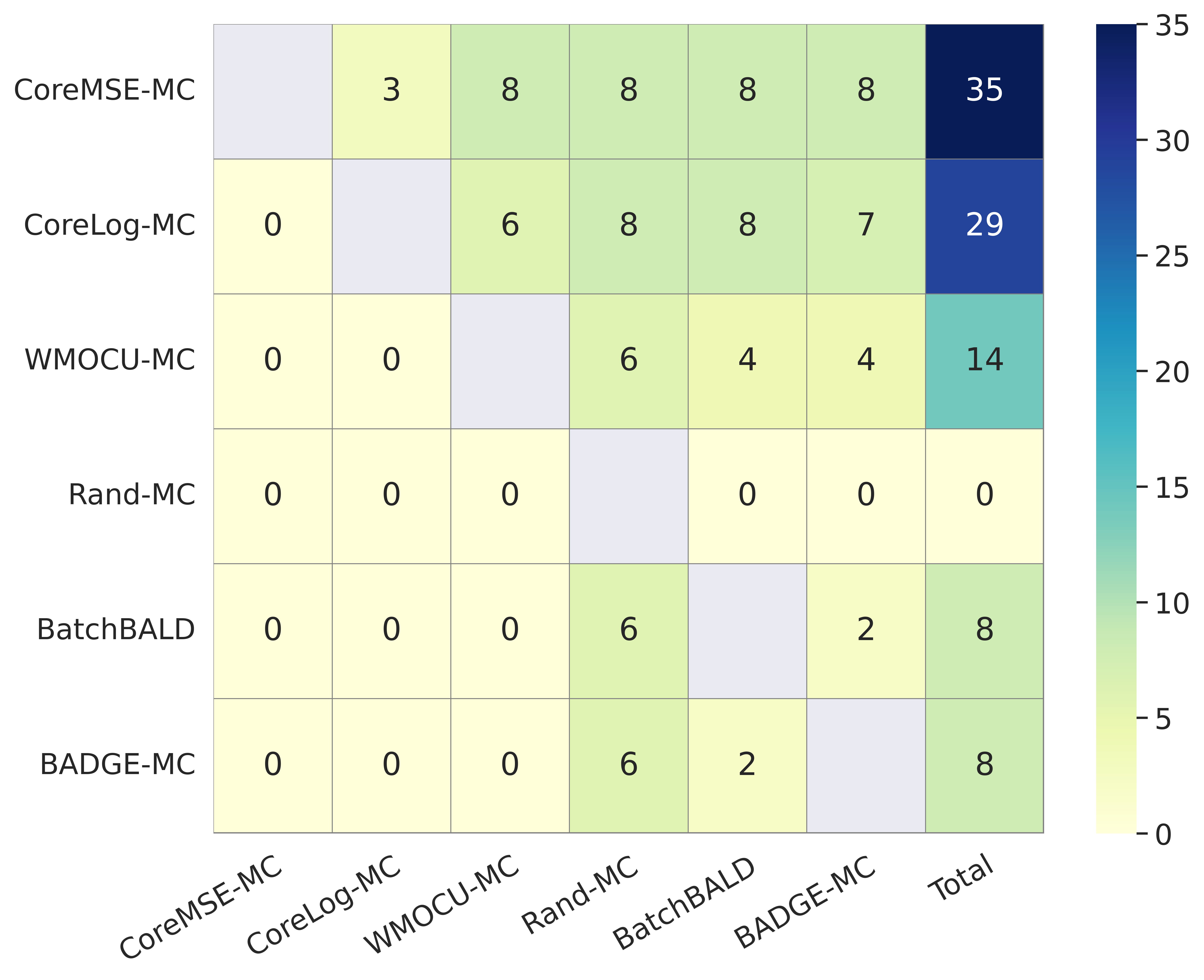}
        \caption{F1-based pairwise comparison}
          \label{fig:figure8_lc_sum_pc_b}
    \end{subfigure}
    \vspace{-2mm}
    \includegraphics[width=0.85\textwidth]{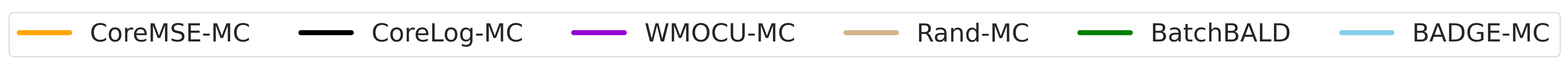}
    \caption{Learning curves and pairwise comparison matrices of batch size 50 for the image datasets based on F1 scores. 
    }
    \label{fig:figure8_lc_sum_pc}
\end{figure*}

\begin{table}[!t]
  \caption{Average training (T) and query (Q) time in seconds, measured after acquiring 500 samples with a batch size of 50 and using 5 ensembles. The times are calculated using five random seeds for various AL methods on DistilBERT}
  \label{tab:table5time}
  \centering 
    {\color{black}
  \begin{tabular}{lllll}
    \toprule    
     AL method &  AG NEWS & PUBMED & IMDB & SST5\\ 
         \cmidrule{2-5}
          &  T / Q & T / Q  & T / Q  & T / Q \\ 
        \cmidrule{2-5}
    Rand  &     75 / 1 & 75 / 1& 120 / 1& 45/1 \\   
    Max-Ent  & 75 / 319  &75 / 52 & 120 / 74 & 45 / 7 \\
    BALD  &  75 / 319    &75 / 47 &120 / 73 & 45 / 4 \\
    MOCU  &   75 / 327   &75 / 48 &120 / 77 & 45 / 6 \\
    WMOCU  &  75 / 333   &75 / 52 &120 / 78 & 45 / 7 \\  
    BADGE  &   75 / 858   &75 / 115 &120 / 130 & 45 / 22 \\ 
    ALPS &   75 / 609  & 75 / 86& 120 / 116 &  45 / 8\\  
    CoreMSE  & 75 / 320  & 75 / 49 & 120 / 74 & 45 / 5  \\
    CoreLog &  75 / 325    &75 / 50 & 120 / 80& 45 / 8 \\
    CoreMSE-MC   &  15 / 320  & 15 / 49 & 24 / 74 & 9 / 5      \\
    CoreLog-MC  &  15 / 325   & 15 / 50 &  24 / 80 &  9 / 8    \\

    \bottomrule
    
  \end{tabular}
}

\end{table}

\begin{table}[!t]
  \caption{Average training time, query time, and F1 score (F1) for ensembles with batch size 50 after acquiring 500 samples. Evaluated using five seeds on DistilBERT for SST5.}
  \label{tab:table5ensemblestime}

  \centering 
  \resizebox{\columnwidth}{!}{
    {\color{black}
    
  \begin{tabular}{lllll}
    \toprule    
     \textbf{\#Ensembles} &  5 & 10 & 20 & 40\\ 
         \cmidrule{1-5}
     \textbf{Metrics} & T/Q/F1  & T/Q/F1  & T/Q/F1  & T/Q/F1 \\ 
        \cmidrule{1-5}
            \textbf{AL Methods} &   &  &  & \\ 
    CoreMSE  & 45/5/0.385  & 90/15/0.391 & 180/33/0.396 & 360/60/0.403  \\
    CoreLog &   45/8/0.384   &90/18/0.390 &180/37/0.395  & 360/68/0.402   \\
    CoreMSE-MC   & 9/5/0.400  & 9/15/0.402 & 9/33/0.404  & 9/60/0.406      \\
    CoreLog-MC  &  9/8/0.396   & 9/18/0.400 &  9/37/0.402  & 9/68/0.405      \\

    \bottomrule
    
  \end{tabular}
}
}

\end{table}

\subsection{Image classification}
As shown in \autoref{sec:ablation},
CoreMSE-MC and CoreLog-MC can significantly reduce training time while maintaining
the same level of performance as their deep ensemble counterparts. 
Therefore, without losing generality,
we compared our CoreMSE-MC and CoreLog-MC with other baselines for the image classification task.

\subsubsection{Datasets} We used two benchmark datasets for image classification: 
MINST and CIFAR10, as shown in \autoref{tab:table3}. 
MNIST \cite{lecun-mnisthandwrittendigit-2010} consists of 28 greyscale images of handwritten digits ranging from 0 to 9. 
The images were divided into 60K and 10K images for training and testing, respectively. 
CIFAR10 \cite{krizhevsky2009learning} consists of 32 × 32 coloured images in 10 classes. Each of the 10 classes has 6K images. 
The training and the test sets consist of 50K and 10K images, respectively. 
These two image datasets have emerged as the preferred choice for many researchers when conducting initial benchmarking of neural networks in image classification tasks.

\subsubsection{Baselines} We considered WMOCU \cite{ZhaoICLR21}, BADGE \cite{ash2019deep} and BatchBALD \cite{KAJvAYG2019} together with a random baseline. 
\ldr{As we utilized those acquisition functions within the context of Algorithm~\ref{alg-ensemble} with MC-Dropout,}
we denote these variants as CoreMSE-MC, CoreLog-MC, WMOCU-MC, BADGE-MC, and Random-MC.
Furthermore, we have also included BatchBALD \cite{KAJvAYG2019} into our baselines as an AL method 
that extends BALD \cite{Houlsby2011} to the batch setting while using MC-Dropout\cite{gal2016dropout}.
Specifically, we followed \cite{KAJvAYG2019} to choose top $B$ informative samples based on their ranked mutual information, which BALD uses to measure how well labeling those samples would improve the the model parameters.
\ldr{We deliberately excluded the ALPS method due to its reliance on a pretrained language model, making it unsuitable for image classification tasks.}





\subsubsection{Model configuration}

\ldr{For the MNIST dataset, we employed a Bayesian Neural Network (BNN) as the backbone classifier. 
BNNs offer enhanced scalability for image data compared to conventional neural networks. We opted for MC-Dropout as the variational approximation method for our ensemble models, primarily due to its effective scalability when dealing with large models and datasets, as indicated in \cite{gal2017deep}.
For the CIFAR10, we utilized PyTorch's \cite{paszke2017automatic} pretrained VGG-16bn,
which incoporates a dropout layer preceding a fully connected layer with 512 hidden units. 
We conducted the experiment over five trials, using five MC-Dropout samples, each with varying batch sizes ($B$ values of 5, 10, 50, and 100). 
All the neural networks were trained for 30 epochs, utilizing the Adam optimizer with an initial learning rate set to 1e-4. 
All experiments were conducted on a workstation equipped with an RTX 3090 GPU and an Intel Core i9 CPU processor.}

\subsubsection{Comparative performance metrics}
Using the performance comparison metrics discussed in \autoref{sssec: modelconfiguration}, 
we generated a pairwise comparison matrix to compare CoreMSE-MC and CoreLog-MC with other baselines. 
\ldr{
The maximum value within each cell of this matrix is $2 \times 4 = 8$, 
reflecting the combination of two datasets and four distinct batch sizes. 
To clarify, if one method consistently outperforms another across both datasets and all four batch sizes, the corresponding cell value will be 8.}

\begin{figure*}[t]
\centering
    \begin{subfigure}[t]{0.39\textwidth}
        \centering
        \includegraphics[width=1\textwidth]{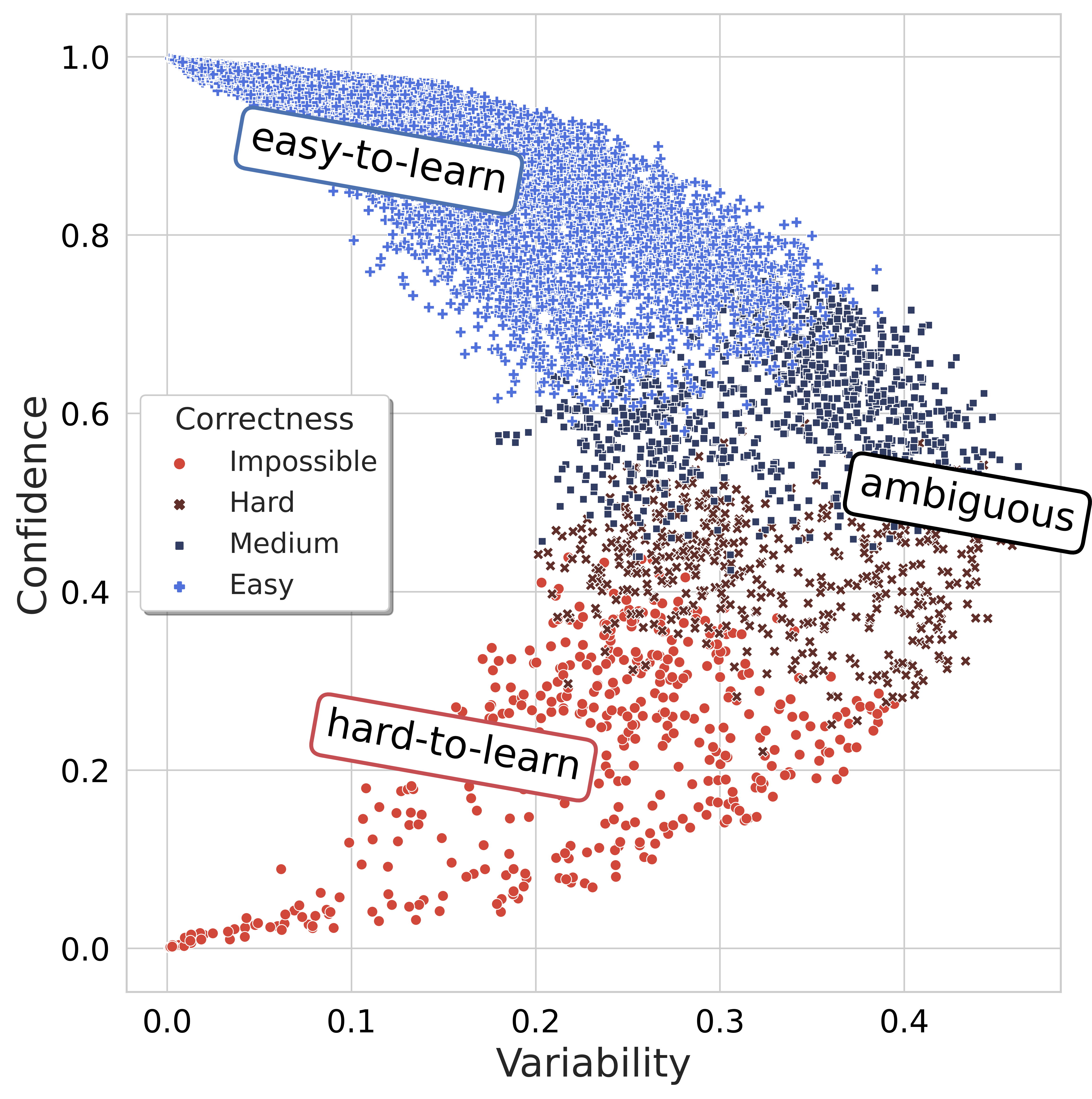}
        \caption{Data Maps}
        \label{fig:datamap_textclassification_a}
    \end{subfigure}%
      \begin{subfigure}[t]{0.59\textwidth}
        \centering
        \includegraphics[width=1\textwidth]{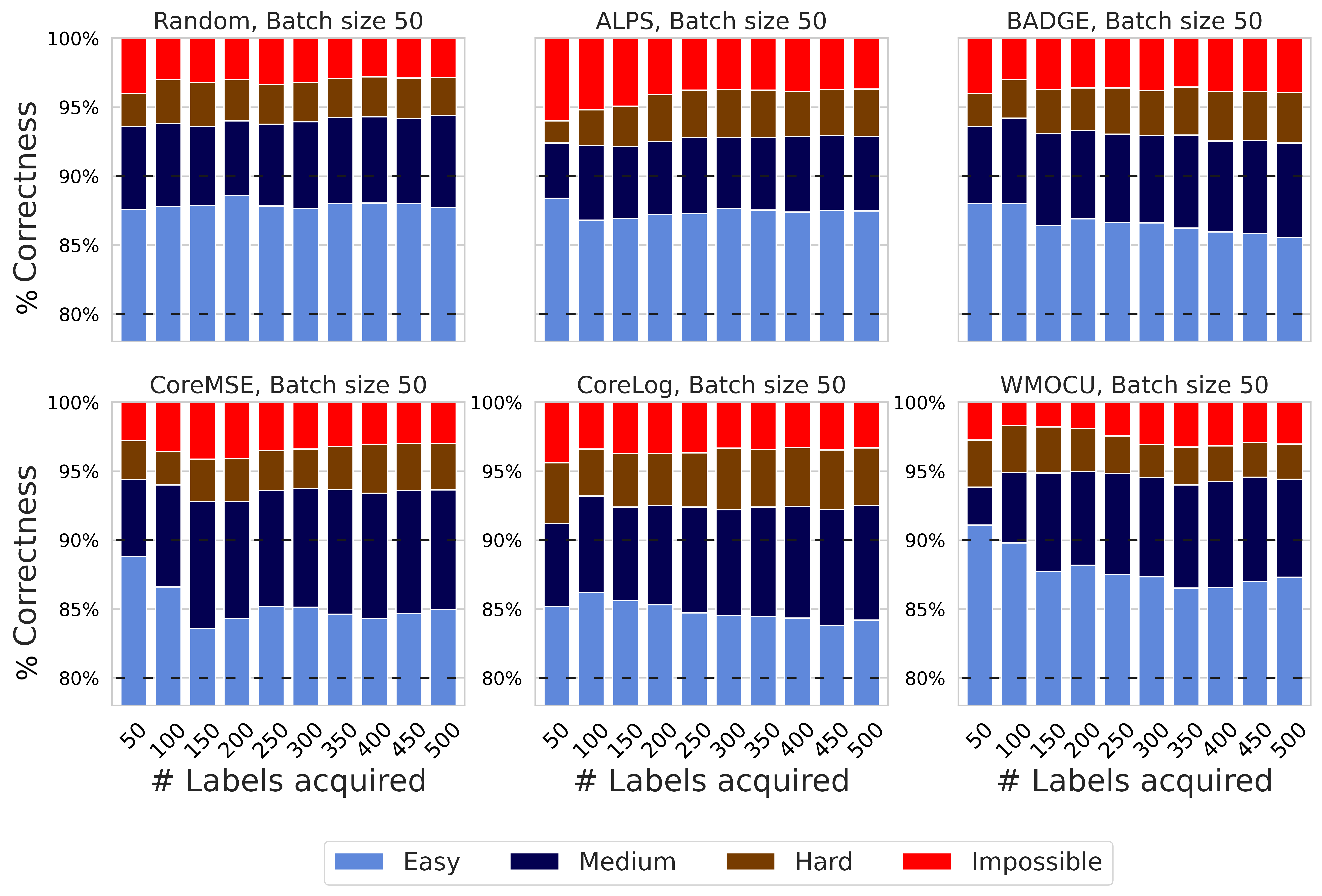}
        \caption{Sampling distribution of ALs with batch size 50 per iteration}
        \label{fig:datamap_textclassification_b}
    \end{subfigure}
    \caption{Learnability of ALs on PUBMED. 
    }
    \label{fig:datamap_textclassification}
\end{figure*}

\subsubsection{Result and Analysis}

\ldr{
As depicted in Fig.~\ref{fig:figure8_lc_sum_pc}, we can observe the performance of these methods, derived on the two image classification datasets. 
The learning curves presented in Fig.~\ref{fig:figure8_lc_sum_pc_a} exhibit similar trends to those observed in balanced datasets such as IMDB and AG NEWS in text classification.}
CoreMSE-MC and CoreLog-MC outperforms all other baselines,
\ldr{with WMOCU-MC being the next most competitive approach, surpassing BatchBALD and BADGE-MC.}
The random baseline consistently performs the worst among all the methods.
\ldr{In addition,
The comparative matrix shown in Fig.~\ref{fig:figure8_lc_sum_pc_b}
supports the observation that CoreMSE-MC and CoreLog-MC significantly outperform the other methods.
Consequently, the scoring functions integrated within the BEMPS framework consistently deliver superior performance compared to alternative methods.}

\begin{figure*}[!t]
\centering
    \centering
    \includegraphics[width=0.8\textwidth]{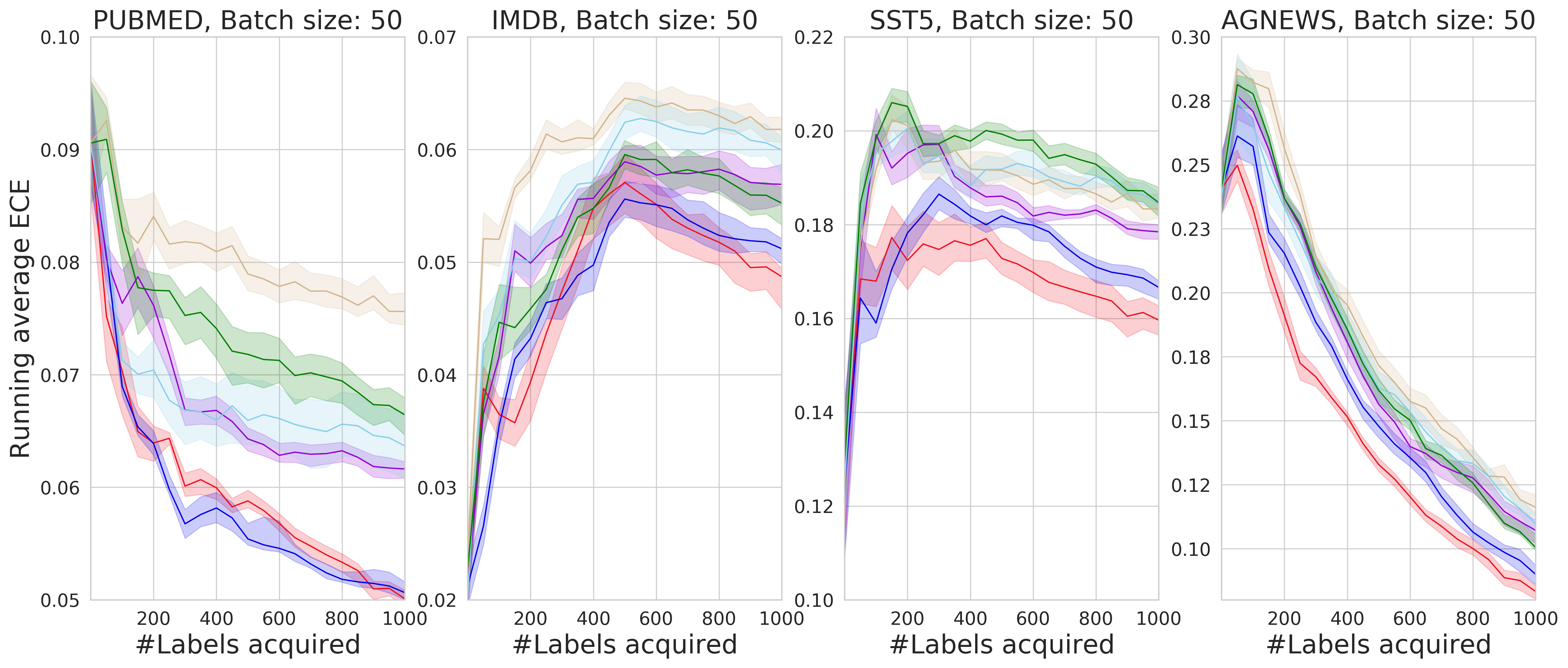}
    \includegraphics[width=0.8\textwidth]{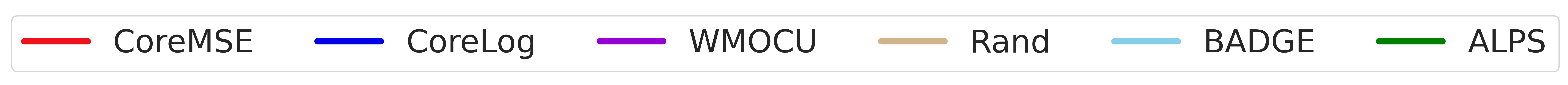}
    \caption{Running average ECE for PUBMED, IMDB, SST5 and AG NEWS.
    }
    \label{fig:runningaverageECE}
\end{figure*}

\begin{figure}[!t]
\centering
        \includegraphics[width=0.485\textwidth]{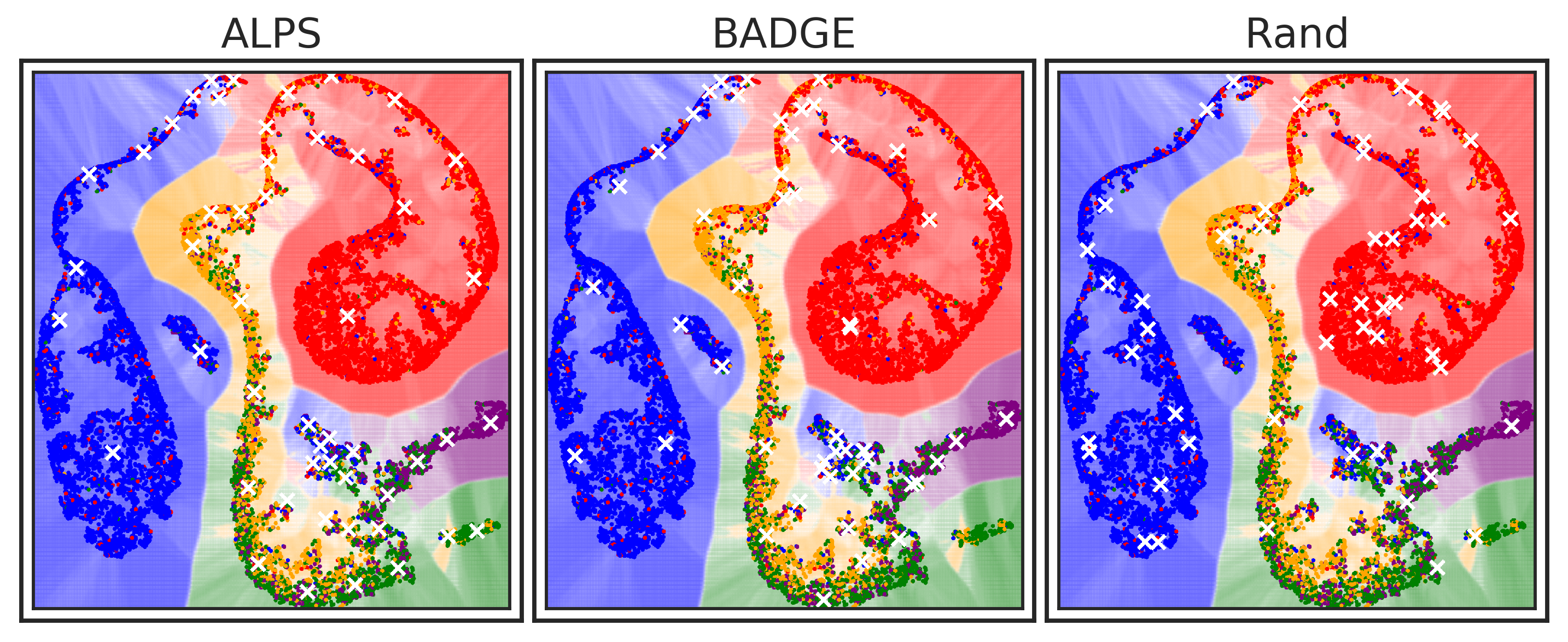}
    \includegraphics[width=0.485\textwidth]{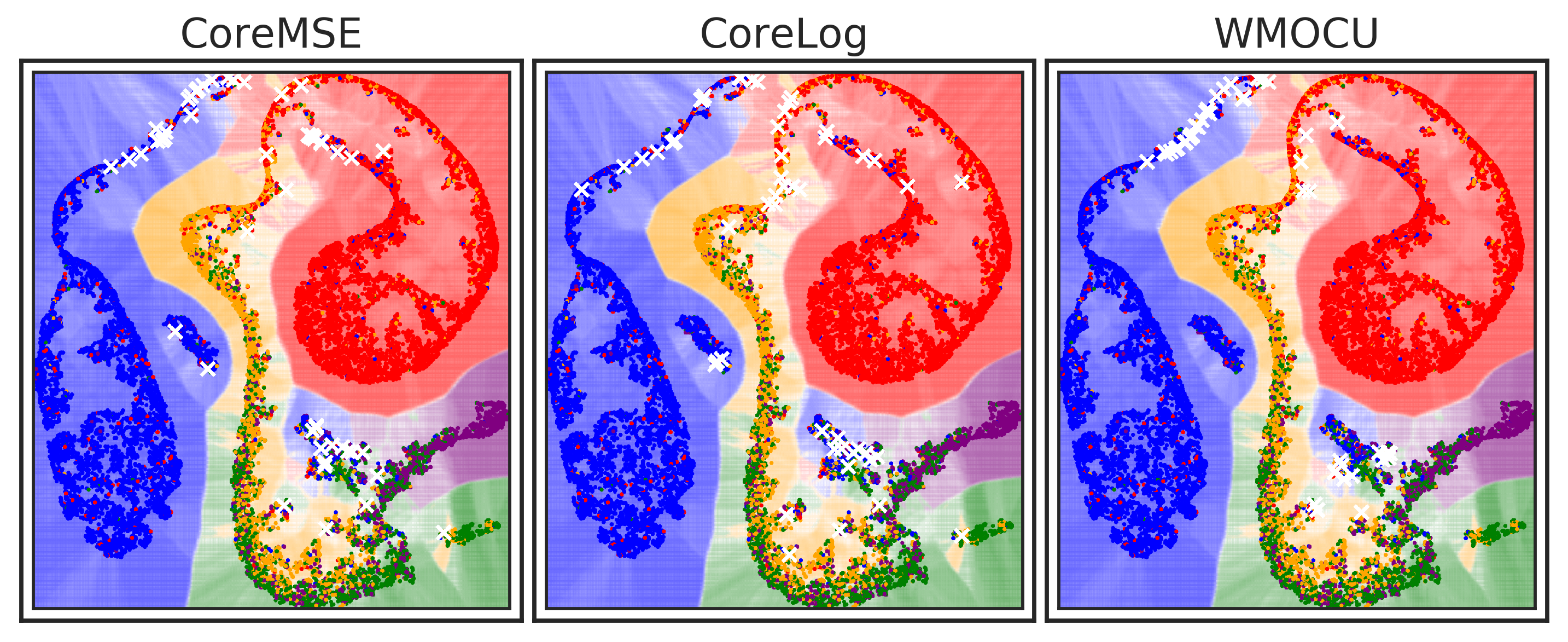}
    \includegraphics[width=0.485\textwidth]{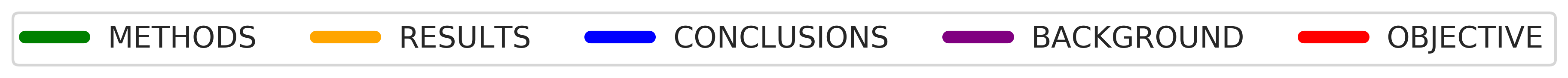}
    \caption{The t-SNE visualization of DistilBERT's decision boundaries for the unlabel pool based on 500 labeled PubMed samples, highlighting 50 specially chosen samples to demonstrate effective diverse batch selection for AL methods.}
    \label{fig:decisionboundary}
\end{figure}

\section{Sampling Behaviours of Different ALs} 
\label{sec:analysis}
\ldr{
In this section, we delved into the sampling behaviours of those AL methods to further understand their performance differences.
We considered
1) using Data Maps \cite{swayamdipta2020dataset} to visualise the sampling behaviours of different AL methods by leveraging the training dynamics;
2) evaluating the calibration of the classifier by computing the expected calibration error method;
and 3) exploring the distribution of selected samples by each AL method along the decision boundary.}

\subsection{Interpretation of Acquisition with Data Maps}
\ldr{
Following \cite{swayamdipta2020dataset}, we created a Data Map for the PUBMED dataset and analyzed the sample selection patterns of different AL methods.
Specifically, 
we used the contextualisation method based on statistics derived from the behavior of the training process, i.e., training dynamics.}

As defined in \cite{swayamdipta2020dataset},
the training dynamics of sample $(\vb{x_i},y_i)$ is defined as statistics derived from $N$ epochs, which is then used to generate coordinates on our map.
The first measure attempts to capture the confidence with which the learner gives the true label to an observation, given its probability distribution.
Confidence is the average model probability of the true label over all epochs: 
$\hat{\mu_i} = \frac{1}{N} \sum_{n = 1}^N p(y_i | \vb*\theta^{(n)}, \vb{x}_i)$ where $\vb*\theta^{(n)}$ indicates the parameters of the classifier derived at the end of the $n^{th}$ epoch, 
and $y_i$ indicates the true label.
The variability measures the spread of $p(y_i | \vb*\theta^{(n)}, \vb{x}_i)$ across epochs: $\hat{\sigma_i} =  \sqrt{\frac{\sum_{n=1}^N (p(y_i | \vb*\theta^{(n)}, \vb{x}_i) - \hat{\mu_i})^2} {N}}$. 
We also considered the correctness by calculating the fraction of instances where the model correctly labels 
$\vb{x}_i$ over multiple epochs.
According to these three dimensions, 
we map these training samples along two axes: the $y$-axis indicates the average model confidence, and the $x$-axis indicates the variability of the samples. 
The map is visualised as a 2D representation of a dataset where samples are distributed on a map based on statistics indicating their ``learnability." 


Both confidence and variability were used to cluster samples into three different regions based on their properties as follows: 1) ``\textbf{easy-to-learn}" samples are those having high confidence and low variability. The classifier consistently predicts those samples accurately; 
2)``\textbf{ambiguous}" samples are those with high variability. In other words,
        the predicted probability of their true classes frequently change during training. Therefore, the classifier is indecisive about those samples;
3) ``\textbf{hard-to-learn}" samples are those with low confidence and low variability, so the classifier often results in poor performance on those samples.
On top of the three regions,
we further categorize the samples into the four types inspired by \cite{karamcheti2021mind} to explore their distribution changes in different AL methods during the sampling iterations.
The correctness score categorizes these four types from 0 to 1: Easy ([0.00, 0.25));
Medium ([0.25, 0.50)); Hard ([0.50, 0.75)); and Impossible ([0.75, 1.00)). 

Fig.~\ref{fig:datamap_textclassification_b} shows CoreMSE and CoreLog acquired more samples in Medium, Hard and Impossible samples than other AL methods over the sampling iterations. 
Thus, CoreMSE and CoreLog can efficiently acquire samples from ``ambiguous" and ``hard-to-learn" regions in Fig.~\ref{fig:datamap_textclassification_a}. 



\subsection{Expected Calibration Error}
Expected Calibration Error (ECE) evaluates the correspondence between predicted probability and empirical accuracy \cite{10.5555/2888116.2888120}.
It places the model's predicted accuracy against its empirical accuracy into bins to estimate the calibration error \cite{pmlr-v70-guo17a, 10.5555/2888116.2888120}.
ECE utilises equal-mass bins and selects the maximum number of bins while maintaining monotonicity in the calibration function. 
Specifically, $ECE = \sum_{m=1}^M \frac{|B_m|}{n} |accuracy(B_m) - confidence(B_m)| $ 
where $n$ is the number of samples and $M$ is interval bins.
The difference of accuracy and confidence values for a set of bins reflects the calibration gap.

\ld{
We used ECE to 
measure how well our text classifier is calibrated.
To do so, we plotted the running average of ECE after each epoch
in Fig.~\ref{fig:runningaverageECE}, where the bin size $M$ was set to 10.
Those plots show that the text classifiers trained with either CoreMSE or CoreLog 
yield better calibration, which can be attributed to that minimising a proper scoring rule should lower ECE \cite{doi:10.1198/016214506000001437}.
It is noteworthy that the classifiers trained using CoreMSE consistently achieved
the lowest ECE across all four text datasets.
This observation suggests that if the primary goal is to improve calibration,
CoreMSE may be the more favourable choice compare to CoreLog.
}

\begin{figure}[!t]
\centering
        \includegraphics[width=0.485\textwidth]{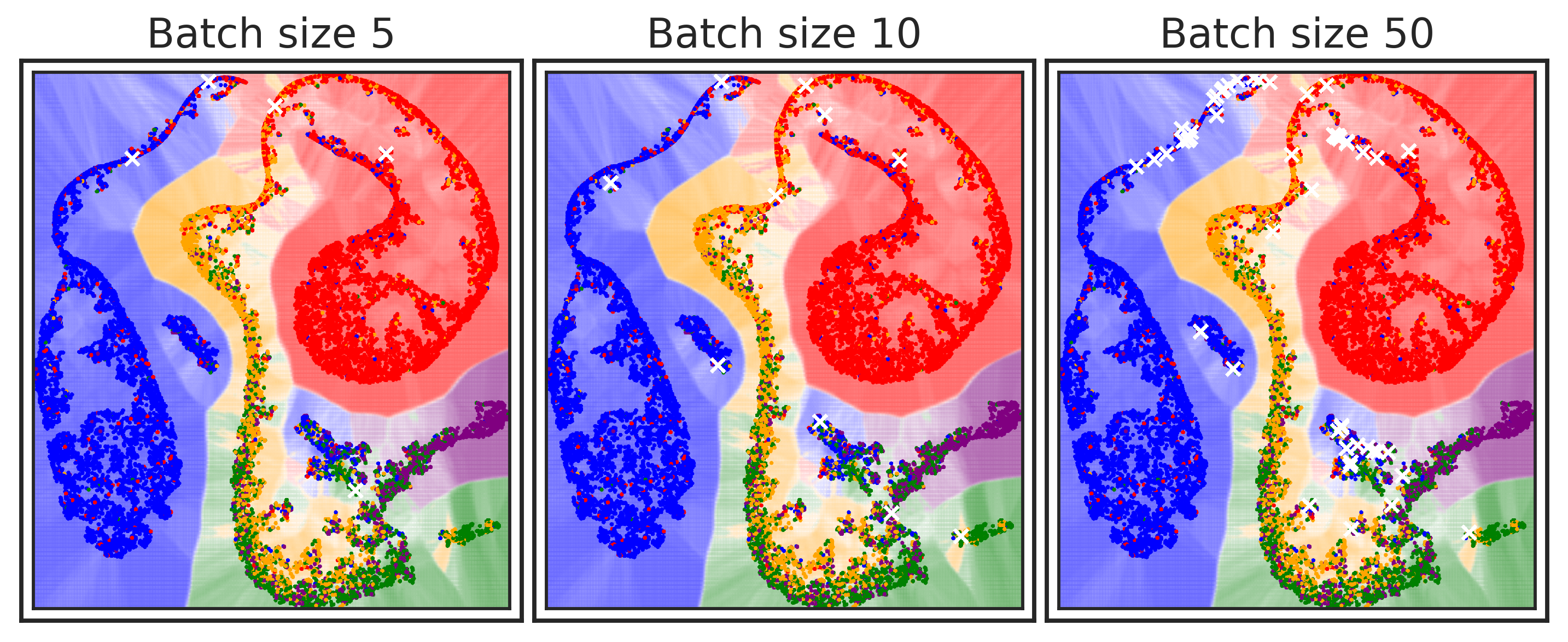}
    \includegraphics[width=0.485\textwidth]{samples/Figures/diversity/pubmed_dlegend.png}
    \caption{t-SNE plot: CoreMSE sample selection at acquisition sizes $B \in {5, 10, 50}$, based on 500 labeled PubMed samples. White crosses indicate acquired samples.}
    \label{fig:decisionboundary2}
\end{figure}

\subsection{Distribution of Acquired Samples}
\wt{
Query strategies focusing exclusively on uncertainty
tend to select similar labels that are close to the decision boundary, which can negatively impact learning outcomes. 
In the t-SNE plots shown in Fig.~\ref{fig:decisionboundary}, 
this issue is clearly illustrated with a white cross marking each selected sample and colored dots indicating their true labels. 
WMOCU, 
which centers on high-uncertainty regions, 
fails to adequately diversify its sample selections, thereby introducing bias and inefficiency. In contrast, ALPS and BADGE opt for a more varied sample set, effectively covering the feature space. 
Similarly, CoreMSE and CoreLog achieve an good balance between selecting uncertain samples and ensuring diversity, highlighting the role diversity plays in the success of AL methods. 
For smaller batch sizes, CoreMSE tends to pick samples with the highest uncertainty from varied cluster regions. However, as the batch size increases, there's a tendency for the algorithm to \ld{repeatedly select samples from the same region, which can convey similar information with respect to the decision boundaries,}
as illustrated in Fig.~\ref{fig:decisionboundary2}. 
In contrast, smaller batch sizes excel at choosing  less-redundant samples from the unlabeled pool in each acquisition round, effectively making their cumulative selections more informative for learning the classifiers.}

\section{Conclusion}
\ld{
We proposed the BEMPS framework for active learning, focusing on 
acquisition functions that are developed based
on strictly proper scoring rules \cite{doi:10.1198/016214506000001437}.
These developed acquisition functions specifically aim to address epistemic uncertainty 
that directly influences the model's classification performance.
For this, we developed convergence theory,
borrowing techniques from \cite{ZhaoICLR21},
that we also extended to the earlier BALD acquisition function.
we  instantiated BEMPS with the Brier score and the logarithmic score,
deriving two new acquisition functions, CoreMSE and CoreLog respectively. 
For more efficient and effective evaluation of those two acquisition functions,
we developed two techniques to improve performance,
a batch AL strategy naturally complement to our BEMPS algorithms,
and a dynamic validation-ensemble hybrid that 
generates high scoring ensembles but does not require a separate large labeled validation set.
}
Empirical results have demonstrated CoreMSE and CoreLog outperformed other AL methods on both text and image classification with different batch sizes.



\ld{
We further investigated both deep ensembles and Monte-Carlo dropout in our BEMPS framework,
aiming to speed up the process of retraining the classifiers after each acquisition iteration.
While traditional deep ensemble method can still work on small datasets,
the stochastic ensemble of Monte-Carlo dropout models can significantly reduce the  cost of 
retraining the classifier.
Nevertheless, both deep ensembles and those based on Monte-Carlo dropout can achieve a similar level of classification performance with the annotation budget.}

\ld{
Finally, we analyzed quantitatively the sampling behaviours of CoreMSE and CoreLog in order 
to understand their performance difference. To do so,
we investigated the distribution of selected samples across three different regions created based on the training dynamics with DataMap, and the distribution along the decision boundary with t-SNE.
Moreover, the ECE showed that the classifier trained with CoreMSE and Corelog
are the most calibrated among those trained with the other acquisition strategies.
It is noteworthy that a limitation of BEMPS and its theory is that it does not support the use of unlabeled data within the learning algorithm, for instance, as would be done by state of the art semi-supervised learning methods.
Given that semi-supervised learning and AL really address a common problem,
this represents an area for future work for our techniques.
}
\ifCLASSOPTIONcompsoc
  \section*{Acknowledgments}
\else
  \section*{Acknowledgment}
\fi

The work has been supported by the Tides Foundation through Grant 1904-57761, as part of the
Google AI Impact Challenge, with Turning Point. Wray Buntine’s work was also supported by
DARPA’s Learning with Less labeling (LwLL) program under agreement FA8750-19-2-0501.

\ifCLASSOPTIONcaptionsoff
  \newpage
\fi


%


\bibliographystyle{IEEEtran}
\bibliography{IEEEabrv,references}

\begin{thebibliography}{10}
\providecommand{\url}[1]{#1}
\csname url@samestyle\endcsname
\providecommand{\newblock}{\relax}
\providecommand{\bibinfo}[2]{#2}
\providecommand{\BIBentrySTDinterwordspacing}{\spaceskip=0pt\relax}
\providecommand{\BIBentryALTinterwordstretchfactor}{4}
\providecommand{\BIBentryALTinterwordspacing}{\spaceskip=\fontdimen2\font plus
\BIBentryALTinterwordstretchfactor\fontdimen3\font minus \fontdimen4\font\relax}
\providecommand{\BIBforeignlanguage}[2]{{%
\expandafter\ifx\csname l@#1\endcsname\relax
\typeout{** WARNING: IEEEtran.bst: No hyphenation pattern has been}%
\typeout{** loaded for the language `#1'. Using the pattern for}%
\typeout{** the default language instead.}%
\else
\language=\csname l@#1\endcsname
\fi
#2}}
\providecommand{\BIBdecl}{\relax}
\BIBdecl

\bibitem{jimenez2018capsule}
A.~Jim{\'e}nez-S{\'a}nchez, S.~Albarqouni, and D.~Mateus, ``Capsule networks against medical imaging data challenges,'' in \emph{Intravascular Imaging and Computer Assisted Stenting and Large-Scale Annotation of Biomedical Data and Expert Label Synthesis}.\hskip 1em plus 0.5em minus 0.4em\relax Springer, 2018, pp. 150--160.

\bibitem{settles2009active}
B.~Settles, ``Active learning literature survey,'' Univ. of Wisconsin--Madison, Computer Sciences Technical Report 1648, 2009.

\bibitem{1642660}
M.~Li and I.~Sethi, ``Confidence-based active learning,'' \emph{{IEEE} Trans. Pattern Anal. Mach. Intell.}, vol.~28, no.~8, pp. 1251--1261, 2006.

\bibitem{4563068}
A.~Holub, P.~Perona, and M.~C. Burl, ``Entropy-based active learning for object recognition,'' in \emph{2008 IEEE Computer Society Conf. Computer Vision and Pattern Recognition Workshops}, 2008, pp. 1--8.

\bibitem{KAJvAYG2019}
A.~Kirsch, J.~van Amersfoort, and Y.~Gal, ``{BatchBALD}: Efficient and diverse batch acquisition for deep {B}ayesian active learning,'' in \emph{Proc. 32nd Advances Neural Inf. Process. Syst.}, 2019.

\bibitem{pop2018deep}
\BIBentryALTinterwordspacing
R.~Pop and P.~Fulop, ``Deep ensemble {B}ayesian active learning : Addressing the mode collapse issue in {M}onte {C}arlo dropout via ensembles,'' 2018. [Online]. Available: \url{https://arxiv.org/abs/1811.03897}
\BIBentrySTDinterwordspacing

\bibitem{yuan-etal-2020-cold}
M.~Yuan, H.-T. Lin, and J.~Boyd-Graber, ``Cold-start active learning through self-supervised language modeling,'' in \emph{Proc. 2020 Conf. Empirical Methods Natural Lang. Process. (EMNLP)}, Nov. 2020, pp. 7935--7948.

\bibitem{sener2017active}
\BIBentryALTinterwordspacing
O.~Sener and S.~Savarese, ``Active learning for convolutional neural networks: A core-set approach,'' 2017. [Online]. Available: \url{https://arxiv.org/abs/1708.00489}
\BIBentrySTDinterwordspacing

\bibitem{schroder2020survey}
\BIBentryALTinterwordspacing
C.~Schröder and A.~Niekler, ``A survey of active learning for text classification using deep neural networks,'' 2020. [Online]. Available: \url{https://arxiv.org/abs/2008.07267}
\BIBentrySTDinterwordspacing

\bibitem{ren2020survey}
P.~Ren \emph{et~al.}, ``A survey of deep active learning,'' \emph{ACM Comput. Surv.}, vol.~54, no.~9, Oct. 2021.

\bibitem{ash2019deep}
J.~T. Ash, C.~Zhang, A.~Krishnamurthy, J.~Langford, and A.~Agarwal, ``Deep batch active learning by diverse, uncertain gradient lower bounds,'' in \emph{Proc. 8th Int. Conf. Learn. Representations}, 2020.

\bibitem{siddhant-lipton-2018-deep}
A.~Siddhant and Z.~C. Lipton, ``Deep {B}ayesian active learning for natural language processing: Results of a large-scale empirical study,'' in \emph{Proc. 2018 Conf. Empirical Methods Natural Lang. Process.}, Oct.-Nov. 2018, pp. 2904--2909.

\bibitem{https://doi.org/10.48550/arxiv.2004.13138}
\BIBentryALTinterwordspacing
J.~Lu and B.~MacNamee, ``Investigating the effectiveness of representations based on pretrained transformer-based language models in active learning for labelling text datasets,'' 2020. [Online]. Available: \url{https://arxiv.org/abs/2004.13138}
\BIBentrySTDinterwordspacing

\bibitem{gao2020cost}
R.~Gao and M.~Saar-Tsechansky, ``Cost-accuracy aware adaptive labeling for active learning,'' in \emph{Proc. AAAI Conf. Artif. Intell.}, 2020, pp. 2569--2576.

\bibitem{zhdanov2019diverse}
\BIBentryALTinterwordspacing
F.~Zhdanov, ``Diverse mini-batch active learning,'' 2019. [Online]. Available: \url{https://arxiv.org/abs/1901.05954}
\BIBentrySTDinterwordspacing

\bibitem{NEURIPS2019_84c2d486}
R.~Pinsler, J.~Gordon, E.~Nalisnick, and J.~M. Hern\'{a}ndez-Lobato, ``Bayesian batch active learning as sparse subset approximation,'' in \emph{Proc. 32nd Advances Neural Inf. Process. Syst.}, 2019.

\bibitem{raiffa1961applied}
H.~Raiffa and R.~Schlaifer, \emph{Applied Statistical Decision Theory}, ser. Harvard Business School Publications.\hskip 1em plus 0.5em minus 0.4em\relax Division of Research, Graduate School of Business Adminitration, Harvard University, 1961.

\bibitem{yoon2013quantifying}
B.-J. Yoon, X.~Qian, and E.~R. Dougherty, ``Quantifying the objective cost of uncertainty in complex dynamical systems,'' \emph{{IEEE} Trans. Signal Process.}, vol.~61, no.~9, pp. 2256--2266, 2013.

\bibitem{10.5555/1046920.1194902}
A.~Banerjee, S.~Merugu, I.~S. Dhillon, and J.~Ghosh, ``Clustering with {B}regman divergences,'' \emph{J. Mach. Learn. Res.}, vol.~6, p. 1705–1749, Dec. 2005.

\bibitem{ZhaoICLR21}
G.~Zhao, E.~Dougherty, B.-J. Yoon, F.~Alexander, and X.~Qian, ``Uncertainty-aware active learning for optimal {B}ayesian classifier,'' in \emph{Proc. 9th Int. Conf. Learn. Representations}, 2021.

\bibitem{zhao2021bayesian}
G.~Zhao, E.~Dougherty, B.-J. Yoon, F.~J.~Alexander, and X.~Qian, ``Bayesian active learning by soft mean objective cost of uncertainty,'' in \emph{Proc. 24th Int. Conf. Artif. Intell. Statist.}, vol. 130, Apr. 2021, pp. 3970--3978.

\bibitem{BEMPS_Wei_NEURIPS2011}
\BIBentryALTinterwordspacing
W.~Tan, L.~Du, and W.~Buntine, ``Diversity enhanced active learning with strictly proper scoring rules,'' in \emph{Advances in Neural Information Processing Systems}, M.~Ranzato, A.~Beygelzimer, Y.~Dauphin, P.~Liang, and J.~W. Vaughan, Eds., vol.~34.\hskip 1em plus 0.5em minus 0.4em\relax Curran Associates, Inc., 2021, pp. 10\,906--10\,918. [Online]. Available: \url{https://proceedings.neurips.cc/paper_files/paper/2021/file/5a7b238ba0f6502e5d6be14424b20ded-Paper.pdf}
\BIBentrySTDinterwordspacing

\bibitem{swayamdipta2020dataset}
S.~Swayamdipta \emph{et~al.}, ``Dataset cartography: Mapping and diagnosing datasets with training dynamics,'' in \emph{Proc. 2020 Conf. Empirical Methods Natural Lang. Process.}, Nov. 2020, pp. 9275--9293.

\bibitem{JMLR:v9:vandermaaten08a}
\BIBentryALTinterwordspacing
L.~van~der Maaten and G.~Hinton, ``Visualizing data using t-sne,'' \emph{Journal of Machine Learning Research}, vol.~9, no.~86, pp. 2579--2605, 2008. [Online]. Available: \url{http://jmlr.org/papers/v9/vandermaaten08a.html}
\BIBentrySTDinterwordspacing

\bibitem{NEURIPS2021_50d2e70c}
G.~Zhao, E.~Dougherty, B.-J. Yoon, F.~Alexander, and X.~Qian, ``Efficient active learning for {G}aussian process classification by error reduction,'' in \emph{Advances Neural Inf. Process. Syst.}, M.~Ranzato, A.~Beygelzimer, Y.~Dauphin, P.~Liang, and J.~W. Vaughan, Eds., vol.~34.\hskip 1em plus 0.5em minus 0.4em\relax Curran Associates, Inc., 2021, pp. 9734--9746.

\bibitem{NEURIPS2021_4afe0449}
J.~Ash, S.~Goel, A.~Krishnamurthy, and S.~Kakade, ``Gone fishing: Neural active learning with {F}isher embeddings,'' in \emph{Advances Neural Inf. Process. Syst.}, M.~Ranzato, A.~Beygelzimer, Y.~Dauphin, P.~Liang, and J.~W. Vaughan, Eds., vol.~34.\hskip 1em plus 0.5em minus 0.4em\relax Curran Associates, Inc., 2021, pp. 8927--8939.

\bibitem{doi:10.1198/016214506000001437}
T.~Gneiting and A.~E. Raftery, ``Strictly proper scoring rules, prediction, and estimation,'' \emph{J. Amer. Statistical Assoc.}, vol. 102, no. 477, pp. 359--378, 2007.

\bibitem{zhu2009active}
J.~Zhu, H.~Wang, B.~K. Tsou, and M.~Ma, ``Active learning with sampling by uncertainty and density for data annotations,'' \emph{{IEEE/ACM} Trans. Audio, Speech, Lang. Process.}, vol.~18, no.~6, pp. 1323--1331, 2009.

\bibitem{zhu2012uncertainty}
J.~Zhu and M.~Ma, ``Uncertainty-based active learning with instability estimation for text classification,'' \emph{ACM Trans. Speech Lang. Process.}, vol.~8, no.~4, pp. 1--21, Feb. 2012.

\bibitem{wang2014new}
D.~Wang and Y.~Shang, ``A new active labeling method for deep learning,'' in \emph{2014 Int. Joint Conf. Neural Netw.}, 2014, pp. 112--119.

\bibitem{gal2017deep}
Y.~Gal, R.~Islam, and Z.~Ghahramani, ``Deep {B}ayesian active learning with image data,'' in \emph{Proc. 34th Int. Conf. Mach. Learn.}, vol.~70, Aug. 2017, pp. 1183--1192.

\bibitem{Houlsby2011}
\BIBentryALTinterwordspacing
N.~Houlsby, F.~Huszár, Z.~Ghahramani, and M.~Lengyel, ``Bayesian active learning for classification and preference learning,'' 2011. [Online]. Available: \url{https://arxiv.org/abs/1112.5745}
\BIBentrySTDinterwordspacing

\bibitem{RoyMcC2001}
N.~Roy and A.~McCallum, ``Toward optimal active learning through sampling estimation of error reduction,'' in \emph{Proc. 18th Int. Conf. Mach. Learn.}, 2001, pp. 441–--448.

\bibitem{WangZengmao2016Abal}
Z.~Wang, B.~Du, L.~Zhang, and L.~Zhang, ``\BIBforeignlanguage{eng}{A batch-mode active learning framework by querying discriminative and representative samples for hyperspectral image classification},'' \emph{\BIBforeignlanguage{eng}{Neurocomputing (Amsterdam)}}, vol. 179, pp. 88--100, 2016.

\bibitem{ma2021active}
\BIBentryALTinterwordspacing
S.~Ma, Z.~Zeng, D.~McDuff, and Y.~Song, ``Active contrastive learning of audio-visual video representations,'' 2020. [Online]. Available: \url{https://arxiv.org/abs/2009.09805}
\BIBentrySTDinterwordspacing

\bibitem{Nguyen2004}
H.~T. Nguyen and A.~Smeulders, ``Active learning using pre-clustering,'' in \emph{Proc. 21st Int. Conf. Mach. Learn.}, 2004, p.~79.

\bibitem{yin2017deep}
C.~Yin \emph{et~al.}, ``Deep similarity-based batch mode active learning with exploration-exploitation,'' in \emph{2017 IEEE Int. Conf. Data Mining}.\hskip 1em plus 0.5em minus 0.4em\relax IEEE, 2017, pp. 575--584.

\bibitem{PengLiu2017ADLf}
P.~Liu, H.~Zhang, and K.~B. Eom, ``\BIBforeignlanguage{eng}{Active deep learning for classification of hyperspectral images},'' \emph{\BIBforeignlanguage{eng}{{IEEE} J. Sel. Topics Appl. Earth Observ. Remote Sens.}}, vol.~10, no.~2, pp. 712--724, 2017.

\bibitem{shi-etal-2021-diversity}
T.~Shi, A.~Benton, I.~Malioutov, and O.~{\.I}rsoy, ``Diversity-aware batch active learning for dependency parsing,'' in \emph{Proc. 2021 Conf. North Amer. Chapter Assoc. Comput. Linguistics: Human Lang. Technologies}, Jun. 2021, pp. 2616--2626.

\bibitem{he2016deep}
K.~He, X.~Zhang, S.~Ren, and J.~Sun, ``Deep residual learning for image recognition,'' in \emph{2016 IEEE Conf. Comput. Vision Pattern Recognit.}, Jun. 2016, pp. 770--778.

\bibitem{krizhevsky2012imagenet}
A.~Krizhevsky, I.~Sutskever, and G.~E. Hinton, ``{ImageNet} classification with deep convolutional neural networks,'' in \emph{Advances Neural Inf. Process. Syst.}, F.~Pereira, C.~Burges, L.~Bottou, and K.~Weinberger, Eds., vol.~25, 2012, pp. 1097--1105.

\bibitem{Lakshminarayanan2017}
B.~Lakshminarayanan, A.~Pritzel, and C.~Blundell, ``Simple and scalable predictive uncertainty estimation using deep ensembles,'' in \emph{Proc. 31st Int. Conf. Neural Inf. Process. Syst.}, 2017, pp. 6405--6416.

\bibitem{pawlowski2017efficient}
\BIBentryALTinterwordspacing
N.~Pawlowski, M.~Jaques, and B.~Glocker, ``Efficient variational {B}ayesian neural network ensembles for outlier detection,'' 2017. [Online]. Available: \url{https://arxiv.org/abs/1703.06749}
\BIBentrySTDinterwordspacing

\bibitem{gal2016dropout}
Y.~Gal and Z.~Ghahramani, ``Dropout as a {B}ayesian approximation: Representing model uncertainty in deep learning,'' in \emph{Proc. 33rd Int. Conf. Mach. Learn.}, vol.~48, Jun. 2016, pp. 1050--1059.

\bibitem{durasov2021masksembles}
N.~Durasov, T.~Bagautdinov, P.~Baque, and P.~Fua, ``Masksembles for uncertainty estimation,'' in \emph{2021 IEEE/CVF Conf. Comput. Vision Pattern Recognit.}, Jun. 2021, pp. 13\,534--13\,543.

\bibitem{savage1971elicitation}
L.~J. Savage, ``Elicitation of personal probabilities and expectations,'' \emph{Journal of the American Statistical Association}, vol.~66, no. 336, pp. 783--801, 1971.

\bibitem{eorms0749}
\BIBentryALTinterwordspacing
R.~L. Winkler and V.~R.~R. Jose, \emph{Scoring Rules}.\hskip 1em plus 0.5em minus 0.4em\relax John Wiley \& Sons, Ltd, 2011. [Online]. Available: \url{https://onlinelibrary.wiley.com/doi/abs/10.1002/9780470400531.eorms0749}
\BIBentrySTDinterwordspacing

\bibitem{merkle2013choosing}
E.~C. Merkle and M.~Steyvers, ``Choosing a strictly proper scoring rule,'' \emph{Decision Analysis}, vol.~10, no.~4, pp. 292--304, 2013.

\bibitem{dawid2014theory}
A.~P. Dawid and M.~Musio, ``Theory and applications of proper scoring rules,'' \emph{Metron}, vol.~72, no.~2, pp. 169--183, 2014.

\bibitem{sanh2019distilbert}
\BIBentryALTinterwordspacing
V.~Sanh, L.~Debut, J.~Chaumond, and T.~Wolf, ``{DistilBERT}, a distilled version of {BERT}: smaller, faster, cheaper and lighter,'' 2019. [Online]. Available: \url{https://arxiv.org/abs/1910.01108}
\BIBentrySTDinterwordspacing

\bibitem{Simonyan2014VeryDC}
\BIBentryALTinterwordspacing
K.~Simonyan and A.~Zisserman, ``Very deep convolutional networks for large-scale image recognition,'' \emph{CoRR}, vol. abs/1409.1556, 2014. [Online]. Available: \url{https://api.semanticscholar.org/CorpusID:14124313}
\BIBentrySTDinterwordspacing

\bibitem{maas2011learning}
A.~L. Maas, R.~E. Daly, P.~T. Pham, D.~Huang, A.~Y. Ng, and C.~Potts, ``Learning word vectors for sentiment analysis,'' in \emph{Proc. 49th Annu. Meeting Assoc. Comput. Linguistics: Human Lang. Technologies}, Jun. 2011, pp. 142--150.

\bibitem{zhang2015character}
X.~Zhang, J.~Zhao, and Y.~LeCun, ``Character-level convolutional networks for text classification,'' in \emph{Proc. 28th Int. Conf. Neural Inf. Process. Syst.}, vol.~1, 2015, pp. 649--657.

\bibitem{dernoncourt2017pubmed}
F.~Dernoncourt and J.~Y. Lee, ``{P}ub{M}ed 200k {RCT}: a dataset for sequential sentence classification in medical abstracts,'' in \emph{Proc. 8th Int. Joint Conf. Natural Lang. Process.}, vol.~2, Nov. 2017, pp. 308--313.

\bibitem{socher2013recursive}
R.~Socher \emph{et~al.}, ``Recursive deep models for semantic compositionality over a sentiment treebank,'' in \emph{Proc. 2013 Conf. Empirical Methods Natural Lang. Process.}, Oct. 2013, pp. 1631--1642.

\bibitem{yang2016active}
Y.~Yang and M.~Loog, ``Active learning using uncertainty information,'' in \emph{Proc. 23rd Int. Conf. Pattern Recognit.}, 2016, pp. 2646--2651.

\bibitem{lewis1994sequential}
D.~D. Lewis and W.~A. Gale, ``A sequential algorithm for training text classifiers,'' in \emph{Proc. 17th Annu. Int. ACM SIGIR Conf. Res. Develop. Inf. Retrieval}, 1994, pp. 3--12.

\bibitem{frankle2018lottery}
J.~Frankle and M.~Carbin, ``The lottery ticket hypothesis: Finding sparse, trainable neural networks,'' in \emph{Proc. 7th Int. Conf. Learn. Representations}, May. 2019.

\bibitem{dodge2020fine}
J.~Dodge, G.~Ilharco, R.~Schwartz, A.~Farhadi, H.~Hajishirzi, and N.~Smith, ``Fine-tuning pretrained language models: Weight initializations, data orders, and early stopping,'' 2020.

\bibitem{loshchilov2017decoupled}
I.~Loshchilov and F.~Hutter, ``Decoupled weight decay regularization,'' in \emph{Int. Conf. Learn. Representations}, 2019.

\bibitem{shah2017simple}
N.~B. Shah and M.~J. Wainwright, ``Simple, robust and optimal ranking from pairwise comparisons,'' \emph{J. Mach. Learn. Res.}, vol.~18, no.~1, pp. 7246--7283, 2017.

\bibitem{zhou2020understanding}
Y.~Zhou, A.~Renduchintala, X.~Li, S.~Wang, Y.~Mehdad, and A.~Ghoshal, ``Towards understanding the behaviors of optimal deep active learning algorithms,'' in \emph{Proc. 24th Int. Conf. Artif. Intell. Statistics}, vol. 130, Apr. 2021, pp. 1486--1494.

\bibitem{lecun-mnisthandwrittendigit-2010}
\BIBentryALTinterwordspacing
Y.~LeCun and C.~Cortes, ``{MNIST} handwritten digit database,'' 2010. [Online]. Available: \url{http://yann.lecun.com/exdb/mnist/}
\BIBentrySTDinterwordspacing

\bibitem{krizhevsky2009learning}
A.~Krizhevsky, ``Learning multiple layers of features from tiny images,'' University of Toronto, Tech. Rep., 2009.

\bibitem{paszke2017automatic}
A.~Paszke, S.~Gross, S.~Chintala, G.~Chanan, E.~Yang, Z.~DeVito, Z.~Lin, A.~Desmaison, L.~Antiga, and A.~Lerer, ``Automatic differentiation in {P}y{T}orch,'' in \emph{Autodiff Workshop, NIPS 2017}, 2017.

\bibitem{karamcheti2021mind}
S.~Karamcheti, R.~Krishna, L.~Fei-Fei, and C.~Manning, ``Mind your outliers! investigating the negative impact of outliers on active learning for visual question answering,'' in \emph{Proc. 11th Int. Joint Conf. Natural Lang. Process.}, vol.~1, Aug. 2021, pp. 7265--7281.

\bibitem{10.5555/2888116.2888120}
M.~P. Naeini, G.~F. Cooper, and M.~Hauskrecht, ``Obtaining well calibrated probabilities using {B}ayesian binning,'' in \emph{Proc. 29th AAAI Conf. Artif. Intell.}, 2015, p. 2901–2907.

\bibitem{pmlr-v70-guo17a}
C.~Guo, G.~Pleiss, Y.~Sun, and K.~Q. Weinberger, ``On calibration of modern neural networks,'' in \emph{Proc. 34th Int. Conf. Mach. Learn.}, vol.~70, Aug. 2017, pp. 1321--1330.

\end{thebibliography}
%



%

\begin{IEEEbiography}[{\includegraphics[width=1in,height=1.25in,clip]{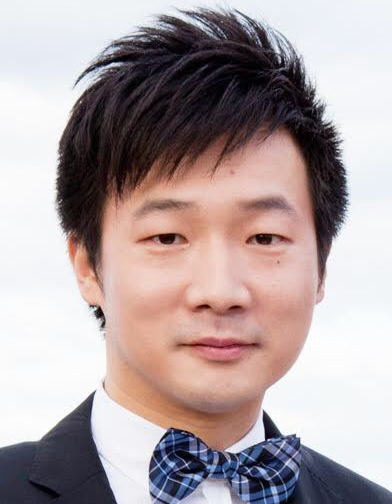}}]{Wei Tan}
Wei Tan is a Doctoral Researcher in the Machine Learning Group at Monash University. He specializes in Active Learning that optimize the labeling budget for the human annotation. He works for the Google Turning point project, which develops a Surveillance System to capture coded ambulance data on SITB, mental health, and AOD attendances to inform policy, strategy, and intervention.
\end{IEEEbiography}
\vfill
\begin{IEEEbiography}[{\includegraphics[width=1in,height=1.25in,clip]{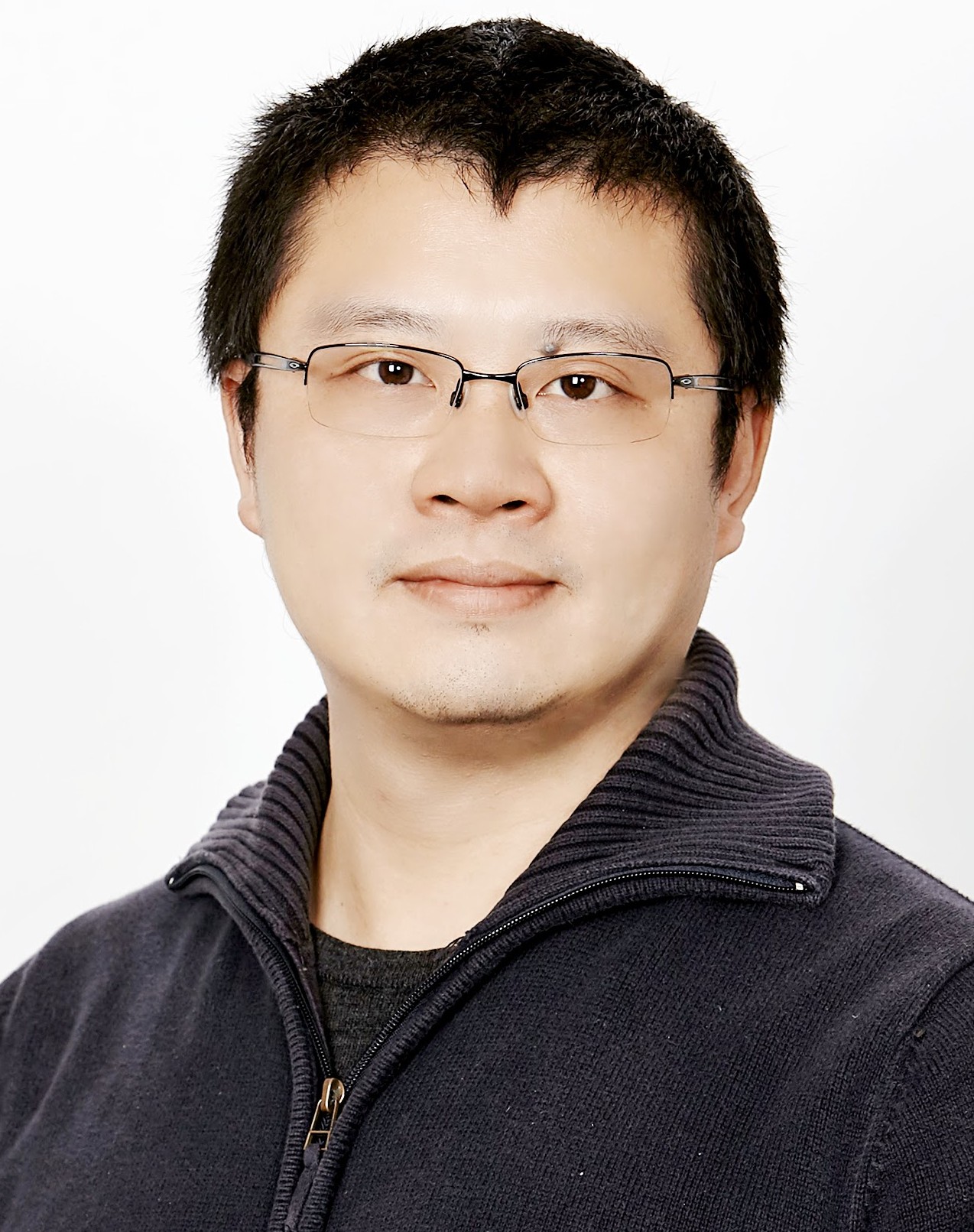}}]{Lan Du}
Dr. Lan Du is currently a senior lecturer in Data Science and AI at Monash University, Australia. His research interest lies in the joint area of machine/deep learning and natural language processing and their cross-disciplinary applications. 
He has published more than 80 high-quality research papers in almost all top conferences/journals.
He has been serving as an editorial board member of the Machine Learning journal and ACM Transactions on Probabilistic Machine Learning, an area chair \& a senior PC of AAAI.
\end{IEEEbiography}


\begin{IEEEbiography}[{\includegraphics[width=1in,height=1.25in,clip]{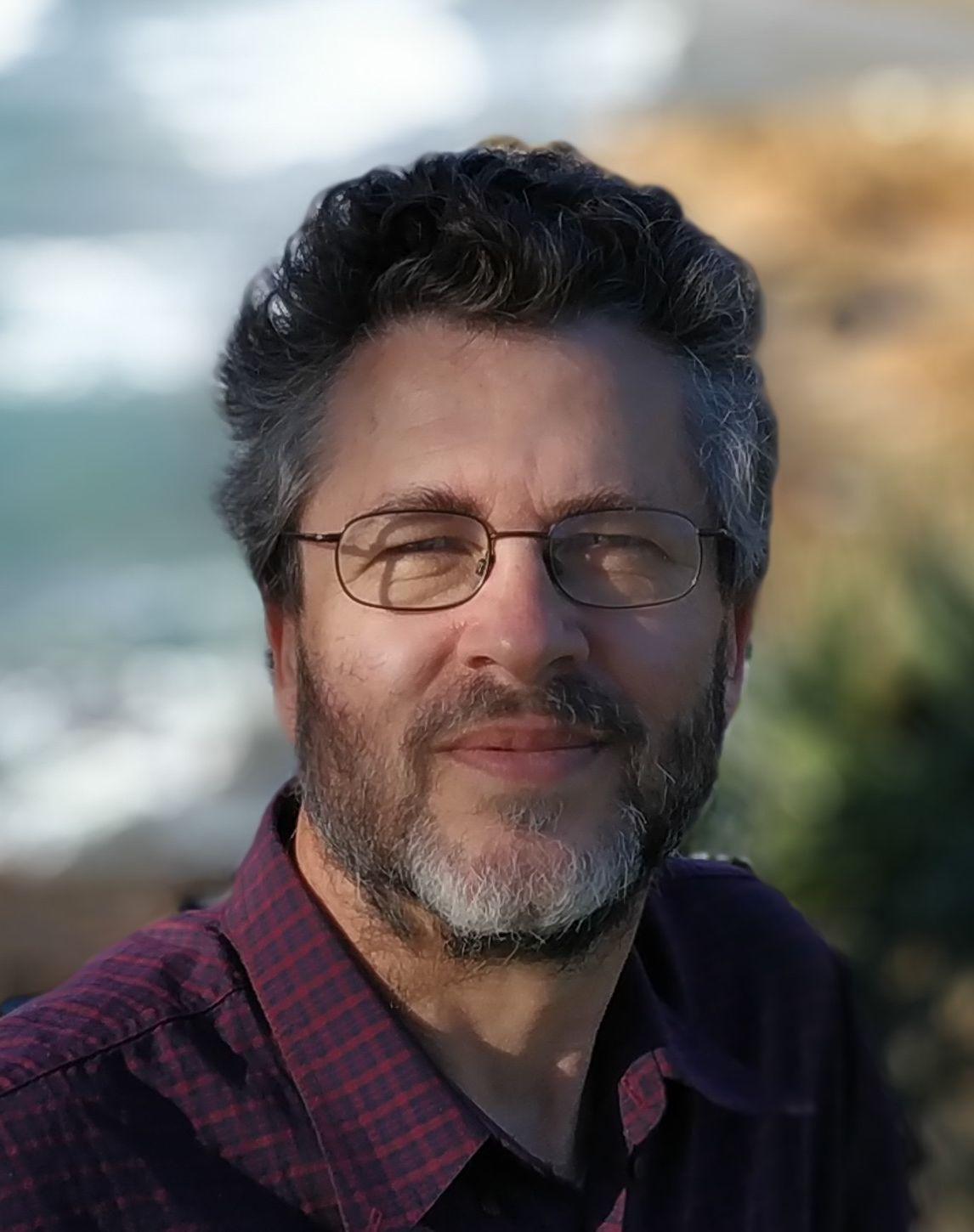}}]{Wray Buntine}
Wray Buntine is full professor and Director of the Computer Science Program at VinUniversity in Hanoi, and was previously Director of the Machine Learning Group at Monash University. He is known for his theoretical and applied work in probabilistic methods for document and text analysis, with over 200 academic publications, 2 patents and some software products.
\end{IEEEbiography}


\vfill



\end{document}